%% file: main.tex
\documentclass[11pt, a4paper, logo, onecolumn, copyright]{pluralisresearchiclr}
\input{macros.tex}
\pdftrailerid{redacted}

\makeatletter
\renewcommand\bibentry[1]{\nocite{#1}{\frenchspacing\@nameuse{BR@r@#1\@extra@b@citeb}}}
\makeatother

\usepackage{kantlipsum, lipsum}
\usepackage{dsfont}
\usepackage{algorithm}
\usepackage{algpseudocode}
\usepackage{adjustbox}
\usepackage{multirow}
\usepackage{nicefrac}

\newcolumntype{R}[2]{%
    >{\adjustbox{angle=#1,lap=\width-(#2)}\bgroup}%
    l%
    <{\egroup}%
}

\usepackage{subcaption}
\usepackage{url}

\graphicspath{{images/}}

\title{AsyncMesh: Fully Asynchronous Optimization for Data and Pipeline Parallelism}

\correspondingauthor{aj@pluralis.ai}

\paperurl{https://github.com/PluralisResearch/AsyncMesh}

\reportnumber{} 

\renewcommand{\today}{} 

\author{Thalaiyasingam Ajanthan, Sameera Ramasinghe, Gil Avraham, Hadi Mohaghegh Dolatabadi, Chamin P Hewa Koneputugodage, Violetta Shevchenko, Yan Zuo, and Alexander Long \\
Pluralis Research
}

\begin{abstract}
    Data and pipeline parallelism are key strategies for scaling neural network training across distributed devices, but their high communication cost necessitates co-located computing clusters with fast interconnects, limiting their scalability. We address this communication bottleneck by introducing {\em asynchronous updates across both parallelism axes}, relaxing the co-location requirement at the expense of introducing {\em staleness} between pipeline stages and data parallel replicas. To mitigate staleness, for pipeline parallelism, we adopt a weight look-ahead approach, and for data parallelism, we introduce an {\em asynchronous sparse averaging} method equipped with an exponential moving average based correction mechanism. We provide convergence guarantees for both sparse averaging and asynchronous updates. Experiments on large-scale language models (up to {\em 1B parameters}) demonstrate that our approach matches the performance of the fully synchronous baseline, while significantly reducing communication overhead.
\end{abstract}

\begin{document}
\maketitle

\input{text}

\bibliographystyle{plainnat}
\bibliography{main}

\appendix
\input{appendix}

\end{document}

%% file: macros.tex
\usepackage{mathtools}
\usepackage{framed}
\usepackage{pbox}
\usepackage{afterpage}
\usepackage{url}
\usepackage{array}
\usepackage{amsfonts}
\usepackage{multirow}
\usepackage{booktabs}
\usepackage{relsize}
\usepackage{xspace}
\usepackage{caption}
\usepackage{subcaption}

\usepackage{csquotes}

\input{math-commands}

\newcommand{\ignore}[1]{}

\newcommand{\norm}[1]{\left\lVert#1\right\rVert}

\usepackage{framed}	
\usepackage{url}
\usepackage{enumerate}
\usepackage{color}

\usepackage{changepage}

\usepackage{algpseudocode}
\usepackage{algorithm}

\usepackage{amsthm}
\usepackage{thmtools}
\declaretheoremstyle[
  headfont=\bfseries,
  notefont=\bfseries, 
  bodyfont=\itshape,
  headformat={\NAME~\NUMBER\NOTE}
]{customthm}

\theoremstyle{definition}
\theoremstyle{customthm}
\newtheorem{thm}{Theorem}
\newtheorem{lem}{Lemma}

\usepackage{stmaryrd}

\newcommand{\SKIP}[1]{}

\newcommand{\bfalpha}{\boldsymbol{\alpha}}

\newcommand{\bfdelta}{\boldsymbol{\delta}}
\newcommand{\bfxi}{\boldsymbol{\xi}}

\renewcommand{\E}{\mathbb{E}\!}

\newcommand{\id}[1]{\!\left\llbracket#1\right\rrbracket}

\newcommand{\NOTE}[1]{[\textbf{\textcolor{blue}{Note:}}\emph{\textcolor{blue}{#1}}]}

\usepackage{xcolor}

\usepackage{bbm}

\usepackage[page]{appendix}
\usepackage{chngcntr}
\usepackage{etoolbox}

\AtBeginEnvironment{subappendices}{%
\section*{Appendix}
\addcontentsline{toc}{section}{Appendix}
}

\usepackage{xspace}
\makeatletter
\DeclareRobustCommand\onedot{\futurelet\@let@token\@onedot}
\def\@onedot{\ifx\@let@token.\else.\null\fi\xspace}

\def\eg{\emph{e.g}\onedot} 
\def\ie{\emph{i.e}\onedot}

\def\iid{i.i.d\onedot}
\makeatother

\usepackage{enumitem}
\newenvironment{tight_itemize}{
\begin{itemize}[leftmargin=10pt]
  \setlength{\topsep}{0pt}
  \setlength{\itemsep}{0pt}
  \setlength{\parskip}{0pt}
  \setlength{\parsep}{0pt}
}{\end{itemize}}

\def\figref#1{Fig.~\ref{#1}}

\def\secref#1{Sec.~\ref{#1}}
\def\eqref#1{Eq.~(\ref{#1})}

\def\tabref#1{Table~\ref{#1}}

\def\thmref#1{Theorem~\ref{#1}}

\def\tabref#1{Table~\ref{#1}}

\algnewcommand\INPUT{\item[\textbf{Input:}]}%
\algnewcommand\OUTPUT{\item[\textbf{Output:}]}%

\usepackage{fontenc}
\usepackage[acronym]{glossaries} 
\makeglossaries 
\glsdisablehyper
\newacronym{MRF}{mrf}{Markov Random Field}
\newacronym{SGD}{SGD}{Stochastic Gradient Descent}
\newacronym{GD}{gd}{Gradient Descent}
\newacronym{PGD}{pgd}{Projected Gradient Descent}
\newacronym{ICM}{icm}{Iterative Conditional Modes}
\newacronym{DNN}{dnn}{Deep Neural Networks}
\newacronym{NN}{nn}{Neural Network}
\newacronym{PMF}{pmf}{Proximal Mean-Field}
\newacronym{PICM}{picm}{Proximal Iterative Conditional Modes}
\newacronym{BWN}{bwn}{Binary Weight Network}
\newacronym{IP}{ip}{Integer Programming}
\newacronym{fc}{fc}{fully-connected}
\newacronym{REF}{ref}{Reference Network}
\newacronym{KL}{kl}{KL}
\newacronym{S}{s}{S}
\newacronym{LR}{lr}{LR}
\newacronym{KKT}{kkt}{KKT}
\newacronym{MAP}{map}{Maximum a Posteriori}
\newacronym{MDA}{mda}{Mirror Descent Algorithm}
\newacronym{MD}{md}{Mirror Descent}
\newacronym{BOP}{bop}{Binary Optimizer}
\newacronym{EDA}{eda}{Entropic Descent Algorithm}
\newacronym{ED}{ed}{Entropic Descent}
\newacronym{EGD}{egd}{Exponentiated Gradient Descent}
\newacronym{PQ}{pq}{ProxQuant}
\newacronym{STE}{ste}{Straight Through Estimator}
\newacronym{HGD}{hgd}{Hybrid Gradient Descent}
\newacronym{PQL}{pql}{PQL}
\newacronym{ELQ}{elq}{Explicit Loss-error-aware Quantization}
\newacronym{ADMM}{admm}{Alternating Direction Method of Multipliers}
\newacronym{QN}{qn}{Quantization Networks}
\newacronym{BR}{br}{BinaryRelax}
\newacronym{FC}{fc}{FC}
\newacronym{XBM}{xbm}{Cross Batch Memory}
\newacronym{SOP}{sop}{Stanford Online Products}
\newacronym{In-shop}{In-shop}{In-shop Clothes Retrieval}
\newacronym{DF2}{DeepFashion2}{Deep Fashion 2}
\newacronym{GPU}{gpu}{GPU}
\newacronym{AWS}{aws}{AWS}
\newacronym{XBN}{xbn}{Cross Batch Normalization}
\newacronym{AXBN}{axbn}{Adaptive Cross Batch Normalization}
\newacronym{TIMM}{timm}{Pytorch Image Models}
\newacronym{PML}{pml}{Pytorch Metric Learning}
\newacronym{MoCo}{MoCo}{Momentum Contrast}
\newacronym{BN}{BN}{Batch Normalization}
\newacronym{LN}{LN}{Layer Normalization}
\newacronym{L2}{L2}{L2}
\newacronym{MBN}{MBN}{Momentum Batch Normalization}
\newacronym{EMA}{EMA}{Exponential Moving Average}
\newacronym{MP}{MP}{Model Parallelism}
\newacronym{PP}{PP}{Pipeline Parallelism}
\newacronym{DP}{DP}{Data Parallelism}
\newacronym{FIM}{FIM}{Fisher Information Matrix}
\newacronym{NAG}{NAG}{Nesterov Accelerated Gradient}
\newacronym{1F1B}{1F1B}{One-Forward-One-Backward}
\newacronym{FFT}{FFT}{Fast Fourier Transform}
\newacronym{RMSE}{RMSE}{Root-Mean-Square Error}
\newacronym{SPARTA}{SPARTA}{Sparse Parameter Averaging}
\newacronym{AsyncPP}{AsyncPP}{Asynchronous Pipeline Parallel}
\newacronym{WT}{WT}{WikiText}
\newacronym{BC}{BC}{BookCorpus}
\newacronym{OWT}{OWT}{OpenWebText}
\newacronym{FW}{FW}{FineWeb}
\newacronym{MoE}{MoE}{Mixture of Experts}

\newcommand{\wiki}{\acrlong{WT}}
\newcommand{\book}{\acrlong{BC}}
\newcommand{\owt}{\acrlong{OWT}}
\newcommand{\fw}{\acrlong{FW}}
\newcommand{\wikis}{\acrshort{WT}}
\newcommand{\books}{\acrshort{BC}}
\newcommand{\owts}{\acrshort{OWT}}
\newcommand{\fws}{\acrshort{FW}}

\newcommand{\asyncpp}{AsyncPP}
\newcommand{\asyncmesh}{AsyncMesh}
\newcommand{\diloco}{DiLoCo}

\newcommand{\asyncsparta}{Async\acrshort{SPARTA}}
\newcommand{\DP}{DPAvg}
\newcommand{\moe}{\acrshort{MoE}}

\usepackage{pifont}

\usepackage{wrapfig}

\usepackage{listings}

\usepackage{xcolor}

\definecolor{codegreen}{rgb}{0,0.6,0}
\definecolor{codegray}{rgb}{0.5,0.5,0.5}
\definecolor{codepurple}{rgb}{0.58,0,0.82}
\definecolor{backcolour}{rgb}{0.95,0.95,0.92}
\definecolor{darkblue}{rgb}{0.0, 0.0, 0.5}

\lstdefinestyle{mystyle}{
  backgroundcolor=\color{backcolour}, commentstyle=\color{codegreen},
  keywordstyle=\color{magenta},
  numberstyle=\tiny\color{codegray},
  stringstyle=\color{codepurple},
  basicstyle=\ttfamily\footnotesize,
  breakatwhitespace=false,         
  breaklines=true,                 
  captionpos=b,                    
  keepspaces=true,                 
  numbers=left,                    
  numbersep=5pt,                  
  showspaces=false,                
  showstringspaces=false,
  showtabs=false,                  
  tabsize=2
}
\renewcommand{\bfdelta}{\boldsymbol{\Delta}}

%% file: math-commands.tex
\usepackage{amsmath,amsfonts,bm}

\def\figref#1{figure~\ref{#1}}

\def\secref#1{section~\ref{#1}}

\def\eqref#1{equation~\ref{#1}}

\def\1{\bm{1}}

\def\rvd{{\mathbf{d}}}
\def\rve{{\mathbf{e}}}

\def\rvg{{\mathbf{g}}}

\def\rvw{{\mathbf{w}}}
\def\rvx{{\mathbf{x}}}

\def\rmD{{\mathbf{D}}}

\def\rmW{{\mathbf{W}}}

\DeclareMathAlphabet{\mathsfit}{\encodingdefault}{\sfdefault}{m}{sl}
\SetMathAlphabet{\mathsfit}{bold}{\encodingdefault}{\sfdefault}{bx}{n}

\def\gB{{\mathcal{B}}}

\def\gD{{\mathcal{D}}}

\def\gO{{\mathcal{O}}}

\def\gS{{\mathcal{S}}}

\newcommand{\etens}[1]{\mathsfit{#1}}

\newcommand{\E}{\mathbb{E}}

\newcommand{\R}{\mathbb{R}}



%% file: text.tex
\section{Introduction}

Distributed optimization approaches enable large-scale model training by partitioning computation across multiple interconnected devices, primarily through \gls{DP}~\cite{goyal2017accurate,li2020pytorch,zhao2023pytorch} and \gls{MP}~\cite{huang2019gpipe,krizhevsky2017imagenet,shoeybi2019megatron}. While \gls{DP} replicates the model across devices with partitioned data, \gls{MP} partitions the model itself across devices. Combining these approaches allows training foundation models at the frontier scale~\cite{dubey2024llama,liu2024deepseek}. However, both \gls{DP} and \gls{MP} rely on high-bandwidth interconnects due to high communication costs, limiting distributed training to co-located computing clusters. Therefore, scaling beyond a centralized infrastructure requires addressing communication bottlenecks inherent to both \gls{DP} and \gls{MP} setups.

Existing approaches mitigate communication costs via compression~\cite{bernstein2018signsgd,wang2023cocktailsgd,wangni2018gradient} and/or overlapping computation with communication using scheduling~\cite{narayanan2019pipedream,qi2023zero} or asynchronous methods~\cite{agarwal2011distributed,stich2019error}. We focus on asynchronous methods, which by design, offer full utilization of the distributed infrastructure and support heterogeneous hardware, by eliminating synchronization barriers. 
We study asynchronous training in a {\em 2D mesh} combining \gls{DP} and \gls{PP}~\cite{huang2019gpipe} -- a special case of \gls{MP} that partitions the model into sequential stages. 

In a synchronous mesh, each stage within a pipe\footnote{A pipe is a sequence of stages that form a full model.} is optimized in a lock-step manner, ensuring that the weights and gradients are synchronized at each step, and the model replicas are explicitly averaged by pausing optimization. In contrast, {\bf \asyncmesh{}} eliminates both synchronization points allowing decoupled optimization in each stage as well as decoupling optimization in a pipe and averaging across replicas. This enables {\em continuous data processing without communication bottlenecks}. \figref{fig:mesh} illustrates this. Asynchronism, however, introduces {\em staleness} in model weights, necessitating correction mechanisms to ensure consensus among model replicas and convergence~\cite{agarwal2011distributed,ajanthan2025asyncpp,stich2019error,zheng2017asynchronous}. 

\begin{figure*}[t]
    \centering
    \vspace{-2ex}
    \includegraphics[width=0.95\linewidth]{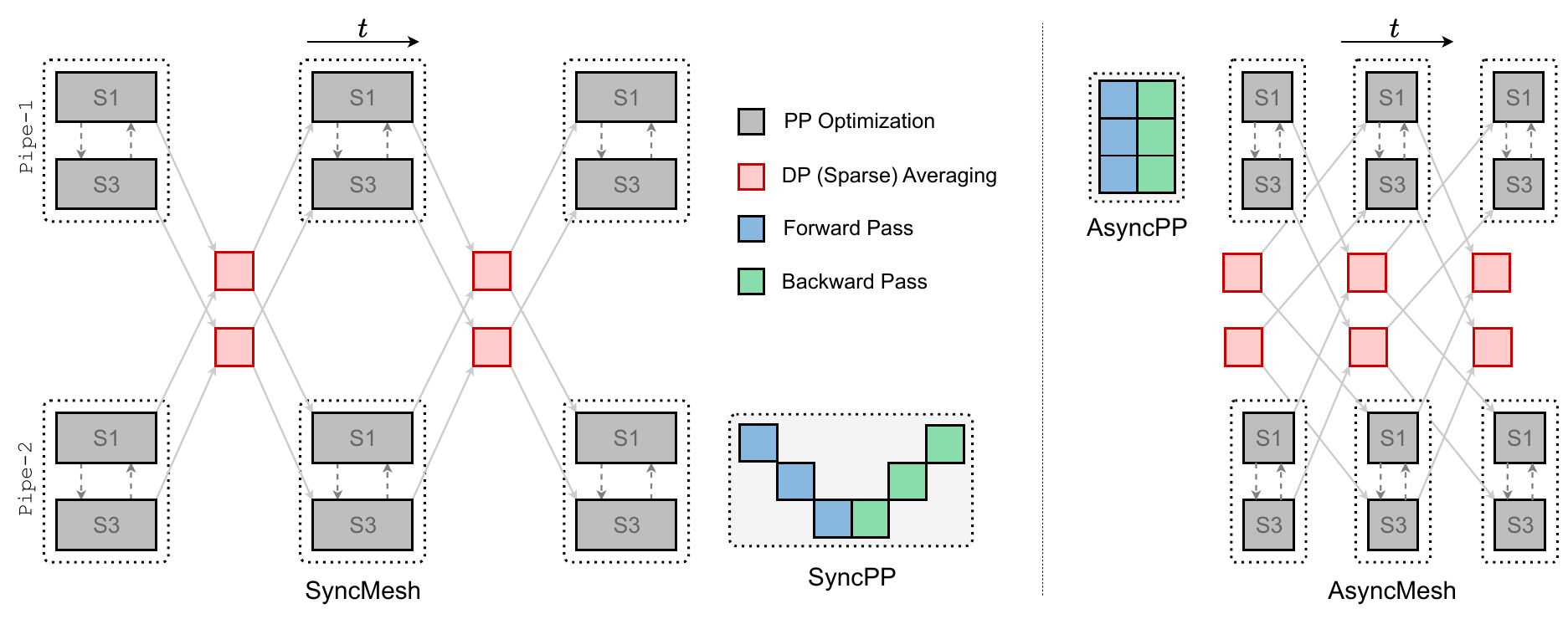}    
    \vspace{-2ex}
    \caption{\em SyncMesh vs AsyncMesh, for a 3-stage, 2-\gls{DP} replica setup (only 2 stages: S1 and S3, are shown for clarity). Notably, AsyncMesh eliminates idle time due to communication for both \gls{PP} and \gls{DP}. In synchronous \gls{DP}, devices are idle while parameters are averaged, whereas asynchronous \gls{DP} eliminates this idle time by using the ``old'' average. Moreover, in Async\gls{PP}, each stage alternates between forward and backward passes without any communication delay. 
    }
    \vspace{-3ex}
    \label{fig:mesh}
\end{figure*}

For \gls{PP}, the staleness is stage-dependent~\cite{narayanan2019pipedream}, and extrapolating the previous update direction using the Nesterov method is shown to effectively compensate for this delay~\cite{ajanthan2025asyncpp}. Due to its simplicity and empirical effectiveness, we adopt this strategy~\cite{ajanthan2025asyncpp} to optimize each pipe asynchronously. For \gls{DP}, the staleness depends on the interconnect bandwidth and the data transfer volume. Therefore, to minimize staleness, we {\em exchange only 5\% of weights} across replicas, similar to sparse averaging~\cite{sparta,fournier2024wash}, and communicate them asynchronously to {\em completely mask the \gls{DP} communication}. To address resulting weight discrepancies, we design an {\em \gls{EMA}} based correction mechanism that approximates the average staleness. Theoretical analysis under the homogeneous setting (identical hyperparameters and \iid data splits) shows that our method ensures consensus among replicas on expectation, despite asynchronous sparse averaging.

We validate our approach on large-scale language modeling tasks with decoder-only transformer architectures~\cite{Karpathy2022,vaswani2017attention}. Experiments demonstrate that our method {\em matches the performance of the fully synchronous baseline}, while eliminating the communication overhead via sparse asynchronous updates. Notably, for the first time, we train a {\em 1B parameter model} to convergence in \asyncmesh{} matching the performance of the synchronous alternative. Our results show the feasibility of distributed training over bandwidth constrained interconnects using asynchronous optimization.

Our contributions can be summarized as follows:
\vspace{-2ex}
\begin{tight_itemize}
\item We introduce a new {\bf \asyncmesh{}} setup where {\em both \gls{DP} and \gls{PP} are asynchronous}, and present a fully asynchronous method that matches the performance of the fully synchronous method.
\item We prove theoretical convergence in the presence of staleness in a homogeneous setup where only a small subset of weights is communicated between \gls{DP} replicas.
\item We empirically show the robustness and scalability of our approach across varying architectures, datasets, subset sizes, staleness levels, \gls{DP} communication intervals, device heterogeneity, and \gls{PP} $\times$ \gls{DP} mesh sizes. 
\end{tight_itemize}

\section{Preliminaries}
We first define the 2D mesh configuration with \gls{DP}~\cite{goyal2017accurate,li2020pytorch} and \gls{PP}~\cite{huang2019gpipe}, and then briefly review the \gls{AsyncPP} method~\cite{ajanthan2025asyncpp} and \gls{SPARTA}~\cite{sparta,fournier2024wash} upon which we build our work. We refer the reader to the respective papers for more details.

\subsection{Problem Setup: 2D Mesh}
Our focus is decentralized training, where the devices are geographically distributed and connected via low-bandwidth interconnects (\eg., the internet), however, our method is applicable for the centralized training setup as well, but its benefits may be limited. To this end, we present the setup and method in an infrastructure agnostic manner.

Let $P$ be the number of pipeline stages and $m$ be the number of replicas, and consider a symmetric 2D mesh, where each stage is replicated $m$ times, constituting $Pm$ devices (or workers) in total. 
A {\em pipe} is a sequence of stages that form the full model, which can be defined using its forward and backward functions. Consider a pipe (or replica) $i\in\{1,\ldots,m\}$, and let the forward function for stage $j$ be 
$f_{ij}\coloneq f_{j}(\rvw_{ij}; \rvx_{j-1})$ with weights $\rvw_{ij}$, and input $\rvx_{j-1}$. Then, the forward and backward functions for pipe $i$ is:
\begin{align}
    F(\rmW_i; \rvx_0) &\coloneqq f_{iP} \circ f_{iP-1}\circ \cdots \circ f_{i1}(\rvx_0)\ ,\ &\mbox{fwd}\ ,\\\nonumber
    G(\rmW_i; \rve_{P}) &\coloneqq g_{i1} \circ g_{i2}\circ \cdots \circ g_{iP}(\rve_{P})\ ,\ &\mbox{bwd}\ ,
\end{align}
where $\rmW_i = \{\rvw_{iP}, \ldots, \rvw_{i1}\}$ and $\rvx_0$ is an input data point.
Here, $g_{ij}\coloneq g_{j}(\rvw_{ij}; \rve_{j})$ is the backward function for stage $j$ corresponding to $f_{ij}$ and $\rve_{P}$ is the error signal corresponding to $\rvx_0$. 
In our 2D mesh, there are $m$ such pipes, and the goal is to optimize the following consensus objective~\cite{boyd2011distributed}:
\begin{align}\label{eq:consensus}
    \min_{\rmW\in\R^d} F(\rmW;\gD) \coloneqq &\min_{\rmW_i\in\R^d}\, \sum_{i=1}^m F(\rmW_i; \gD_i)\ ,\\\nonumber
    &\text{s.t.}\quad \rmW_i = \rmW\ ,\quad\forall\,i\in \{1,\ldots,m\}\ ,
\end{align}
where $\gD_i$ is an \iid subset\footnote{Unlike federated learning, \iid data splits is a realistic assumption in decentralized training.} of the dataset $\gD$ and $d$ is the number of learnable parameters of the full model.
Each pipe is optimized independently on its own data split and using its own optimizer, and its weights are synchronized periodically with other pipes -- typically after every optimization step. 

\subsection{\acrlong{PP}}
\acrfull{PP} methods~\cite{pp-survey} provide communication efficient ways to optimize the model parameters in a pipe. Specifically, pipeline scheduling strategies such as GPipe~\cite{huang2019gpipe}, \acrshort{1F1B}~\cite{narayanan2021efficient}, and ZeroBubble~\cite{qi2023zero} design the order of processing forward and backward passes of microbatches in a pipe to reduce communication overhead between stages and improve device utilization. In contrast, asynchronous \gls{PP} methods~\cite{ajanthan2025asyncpp,narayanan2019pipedream} eliminate the requirement to synchronize the weights and gradients across stages at each update step, offering full pipeline utilization.\footnote{Asynchronous \gls{PP} methods~\cite{ajanthan2025asyncpp,pipedream-2bw} assume comparable compute and communication times per stage, but in bandwidth limited scenarios, compression~\cite{ramasinghe2025protocolmodelsscalingdecentralized} is needed to fully eliminate \gls{PP} communication overhead.}

\vspace{-1ex}
\paragraph{\acrlong{AsyncPP}.}\label{sec:asyncpp}
The main challenge in the \acrfull{AsyncPP} method~\cite{ajanthan2025asyncpp} is the discrepancy between the gradients at a particular stage and the corresponding weights due to asynchronous updates. Specifically, since the weights of a particular stage are updated multiple times between the forward and backward passes of a microbatch, outdated gradients are used for weight updates. Formally, the weight update for pipe $i$ and stage $j$ can be written as (omitting optimizer specifics):
\vspace{-1ex}
\begin{equation}
    \rvw_{ij}^{t+1} = \rvw_{ij}^t - \eta_t\,\nabla f_{j}(\rvw_{ij}^{t-\delta_j};\gB_i^{t-\delta_j})\ ,
    \label{eq:gd}
\end{equation}
where $\eta_t > 0$ is the learning rate, $\gB_i^{t-\delta_j}$ is the minibatch, and $\delta_j \ge 0$ is the stage-dependent constant delay due to asynchronous \gls{PP} updates. 

To compensate for this delay, \asyncpp{}~\cite{ajanthan2025asyncpp} extrapolates the last update direction ${(\rvw_{ij}^t - \rvw_{ij}^{t-1})}$ using the \gls{NAG} framework~\cite{nesterov1983method}, and shows that it acts as gradient delay correction in the weight space. This was shown to surpass the synchronous GPipe method on some language modeling tasks, and we employ this approach to optimize each pipe in our experiments.

\subsection{\acrlong{DP}}
In a typical \acrfull{DP} setup~\cite{goyal2017accurate,li2020pytorch}, each worker computes the gradient of the full model on a minibatch and communicates it to the central parameter server. The server averages the gradients from all workers, performs an optimization step of the server model, and distributes the updated parameters to the workers for the next iteration. This is equivalent to performing gradient based optimization using a larger minibatch, and the consensus constraint in \eqref{eq:consensus} is maintained. 

We consider a serverless scenario, which better reflects a decentralized training setup~\cite{ryabinin2023swarm}. In this, each worker performs a local update using its optimizer and averages the updated weights with others~\cite{mcmahan2017fedavg}, or equivalently averages gradients before applying local updates~\cite{ryabinin2021moshpit}. Synchronizing weights (instead of gradients) is becoming popular as they can be communicated infrequently to reduce data transfer~\cite{douillard2025streaming,douillard2023diloco,reddi2020fedopt}. To this end, we consider the setup where only a small subset is averaged periodically~\cite{sparta,fournier2024wash,lee2023partial} which performs similarly to traditional \gls{DP}.

\vspace{-1ex}
\paragraph{\acrlong{SPARTA}.}
In \acrfull{SPARTA}~\cite{sparta,fournier2024wash}, after each local update, a small subset of parameters (\eg, 5\%) is averaged across workers, reducing data transfer. Adopting this to our 2D mesh is straightforward, as the local update is performed on each pipe, and the sparse averaging is done between stage replicas with a randomly sampled subset. 
Formally, sparse averaging for stage $j$ with subset $\gS^t_j\subset \{1,\ldots,d_j\}$ can be written as (omitting optimizer specifics):
\begin{align}\label{eq:sparta}
    \hat{\rvw}_{ij}^{t} &= \rvw_{ij}^{t-1} - \eta_{t-1}\,\nabla f_{j}(\rvw_{ij}^{t-1};\gB_i^{t-1})\ ,\qquad\forall\,i\ ,\\\nonumber 
    w_{ij:\mu}^{t} &= \left\{\begin{array}{ll}
\frac{1}{m}\sum_i\hat{w}_{ij:\mu}^{t}\ , \quad&\mbox{if $\mu\in \gS^t_j$}\ ,\\[1ex]
\hat{w}_{ij:\mu}^{t}\ , \quad&\mbox{if $\mu\notin \gS^t_j$}\ ,\end{array} \right.\qquad\ \ \ \forall\,i,\mu\ .  
\end{align}
Here, $\eta_t > 0$ is the learning rate, $\gB_i^{t-1}$ is the minibatch, 
and $w_{ij:\mu}^t$ is the $\mu$-th element of vector $\rvw_{ij}^t$. 


\section{Our Method: Optimization in \asyncmesh{}}
We consider \asyncmesh{}, where both \gls{PP} and \gls{DP} axes in the consensus optimization problem (\eqref{eq:consensus}) are optimized asynchronously.
By making the mesh fully asynchronous, our setup ensures {\em full pipeline utilization} throughout training without interruption, while encouraging consensus among model replicas. Our setup is clearly illustrated in \figref{fig:mesh}.
For optimizing each pipe, we employ the \gls{AsyncPP} method (refer to \secref{sec:asyncpp}) that uses a variant of \gls{NAG} for delay correction within a pipe. For \gls{DP}, we introduce an asynchronous version of sparse parameter averaging that eliminates the communication overhead due to \gls{DP}.

\subsection{Asynchronous \acrlong{SPARTA}}
Let us consider a particular stage, and drop the stage index for simplified notation.
In asynchronous \gls{DP}, the local updates in each worker (\ie, \gls{PP} optimization) do not wait for the averaging operation (\ie, \gls{DP} communication) to complete. Therefore, the weights at each stage are updated multiple times while the averaging operation is performed between the stage replicas. Thus, the averaged weights are older than the weights at each worker, which leads to {\em staleness}. Let $\tau$ be the corresponding delay, then, the delayed sparse averaging can be written as:
\begin{align}\label{eq:asparta}
    w_{i:\mu}^{t} &= \left\{\begin{array}{ll}
\bar{w}_{i:\mu}^{t-\tau}\ , \quad&\mbox{if $\mu\in \gS^{t-\tau}$}\ ,\\[1ex]
\hat{w}_{i:\mu}^{t}\ , \quad&\mbox{if $\mu\notin \gS^{t-\tau}$}\ ,\end{array} \right.\qquad\ \ \ \forall\,i,\mu\ ,
\end{align}
where $\bar{\rvw}^t = \frac{1}{m}\sum_i\hat{\rvw}_{i}^{t}$ denotes the averaged weights.
Compared to \eqref{eq:sparta}, the only difference is that elements in the subset $\gS^{t-\tau}$ are set to the {\em old average} $\bar{w}_{i:\mu}^{t-\tau}$ instead of the new one at time $t$. This delay is detrimental to training as 1) it ignores the local updates between $t\!-\!\tau$ and $t$ for the subset, and 2) the weight vector has a discrepancy as some weights correspond to time $t\!-\!\tau$ and the rest at time $t$.
Therefore, it is essential to compensate for this delay.

\subsection{Delay Correction via Estimating Average Staleness}\label{sec:ema}
Our idea is to approximate the new average $\bar{\rvw}^t$ using the old average $\bar{\rvw}^{t-\tau}$ and the estimated {\em average staleness}. Specifically, we estimate the average staleness in each stage independently using \acrfull{EMA} of staleness (\ie, weight differences) throughout training. Precisely, we estimate the new average as,
\begin{align}\label{eq:ema}
    \rvd_{i}^t\id{\gS^{t-\tau}} &= (1-\lambda_t)\, \rvd_{i}^{t-1}\!\id{\gS^{t-\tau}} + 
    \lambda_t\left(\hat{\rvw}_{i}^t\id{\gS^{t-\tau}} - \hat{\rvw}_{i}^{t-\tau}\!\id{\gS^{t-\tau}}\right)\ ,\ &\mbox{\gls{EMA}}\ ,\\\nonumber
    \tilde{\rvw}_i^{t}\id{\gS^{t-\tau}} &= \bar{\rvw}^{t-\tau}\!\id{\gS^{t-\tau}} + \rvd_i^t\id{\gS^{t-\tau}}\ ,\ &\mbox{est. avg.}\ ,
\end{align}
where $\id{\cdot}$ denotes the indicator and $\lambda_t\in(0,1)$ is the \gls{EMA} coefficient. Then, the delay corrected asynchronous sparse averaging takes the following form:
\begin{equation}
    w_{i:\mu}^{t} = \left\{\begin{array}{ll}
\tilde{w}_{i:\mu}^{t}\ , \quad&\mbox{if $\mu\in \gS^{t-\tau}$}\ ,\\[1ex]
\hat{w}_{i:\mu}^{t}\ , \quad&\mbox{if $\mu\notin \gS^{t-\tau}$}\ ,\end{array} \right.\qquad\ \ \ \forall\,i,\mu\ .
\end{equation}
Intuitively, the idea here is that $\rvd_i^t$ being the \gls{EMA} of staleness, robustly estimates the average staleness. Therefore, $\tilde{\rvw}_i^t$ approximates the average at time $t$. Concretely,
\begin{align}
    \tilde{\rvw}_i^{t} = \bar{\rvw}^{t-\tau} + \rvd_i^t &\approx^{\etens{a}} \bar{\rvw}^{t-\tau} + \E\left[\hat{\rvw}_{i}^t - \hat{\rvw}_{i}^{t-\tau}\right] \\\nonumber &\approx^{\etens{b}} \bar{\rvw}^{t-\tau} + \bar{\rvw}^{t} - \bar{\rvw}^{t-\tau} = \bar{\rvw}^{t}\ .
\end{align}
Here, the approximation $\etens{a}$ is due to \gls{EMA} being a stochastic approximation of the expected value~\cite{robbins1951stochastic}, and $\etens{b}$ is due to the empirical average of \gls{DP} replicas.
Note the accuracy of the \gls{EMA} approximation depends on the delay $\tau$ (lower the better) and the smoothness of the optimization trajectory.\footnote{Since \gls{EMA} is estimated across many time steps, a smoother trajectory (\ie, slow drift in average staleness) leads to a better approximation, conditions for this are stated in the next section.} Whereas the approximation error of $\etens{b}$ reduces with increasing number of \gls{DP} replicas.

Note that sparse averaging is performed after every local step in our description so far. However, this requirement can be relaxed and averaging can be performed every $K$ steps, similar to~\cite{douillard2023diloco,mcmahan2017fedavg}. In this setup, our approach is a strict generalization of the concurrent idea of eager \diloco{}~\cite{kale2025eager}, in which, all the parameters are communicated, the delay $\tau=K$, and ${\rvd_i^t = \frac{1}{m}\left(\hat{\rvw}_{i}^t - \hat{\rvw}_{i}^{t-\tau}\right)}$.

\subsection{Theoretical Insights}
Recall that our aim is to optimize the consensus objective \eqref{eq:consensus} with asynchronous updates for both \gls{PP} and \gls{DP}. \asyncpp{}~\cite{ajanthan2025asyncpp} proved convergence with fixed delay, providing a theoretical justification for the single pipeline setup. However, it is unclear if sparse averaging ensures consensus, and since we make it asynchronous, the impact of staleness also needs to be studied.

To this end, we take a first step in theoretically understanding the conditions required for achieving consensus for asynchronous sparse averaging. First, we show that sparse averaging ensures consensus on expectation if the learning rate is chosen proportional to the subset size. Then, we provide a theoretical insight showing that 
under standard assumptions of stochastic approximation~\cite{robbins1951stochastic} with an identical setup for each replica, the 
\gls{EMA} approximates the average staleness and consequently ensures consensus on expectation. Both of these provide a theoretical justification that our approach ensures consensus, and when coupled with the standard convergence proof for \gls{SGD}~\cite{bottou2018optimization}, show that our approach converges to a fixed point of \eqref{eq:consensus}.

Suppose the consensus error is defined as:
\vspace{-1ex}
\begin{equation}
    \norm{\bfdelta^t}^2 \coloneqq \sum_{i=1}^m \norm{\rvw_i^t - \rvw^t}^2\ ,\ \mbox{where $\rvw^t = \frac{1}{m}\rvw_i^t$}\ .\\[-1.5ex]
\end{equation}
We first show that, on expectation, the consensus error vanishes, \ie, all model replicas converge to their average. Now, by invoking the standard convergence proof for \gls{SGD}~\cite{bottou2018optimization} on the averaged model, one can show convergence for sparse parameter averaging.
\begin{thm}[Sparse averaging ensures consensus]\label{thm:sparta}
    Let $f$ be a $L$-smooth function, the stochastic gradient be an unbiased estimate of $\nabla f$ and have bounded variance, and $p>0$ be the averaging probability for an element ${\mu\in\{1,\ldots,d\}}$, then, for an appropriate choice of learning rate $\eta_t>0$, the consensus error for updates in \eqref{eq:sparta} diminishes on expectation, \ie, $\lim_{t\to \infty}\E\left[\norm{\bfdelta^{t}}^2 \right] = 0$.
\end{thm}
\vspace{-2ex}
\begin{proof}
    We first show that sparse averaging shrinks the consensus error by a factor of $(1\!-\!p)$ on expectation. Then, we derive a recursion on $\E\left[\norm{\bfdelta^{t}}^2 \right]$ and choose a learning rate  proportional to $p$ to ensure that it is a contraction. Detailed proof is provided in the appendix.
\end{proof}
\vspace{-2ex}
Consensus of partial averaging has been previously studied in federated learning for pre-determined subsets~\cite{lee2023partial}, where a similar relationship between the learning rate and subset size is observed. This suggests convergence may be slow for small subset sizes (equivalently, small $p$) due to small learning rates. However, the tightness of this result is unclear, and we leave any such analysis for future work. In practice, we use the standard learning rate value and do {\em not} adjust it based on the subset size.

For the theoretical analysis of asynchronous sparse averaging, we consider a {\em homogeneous setup}\footnote{Following standard practice we use this simplifying assumption for the theoretical analysis, however, we provide experiments in the  heterogeneous setup showing the merits of our method.} with the same initialization, optimizer hyperparameters, and \iid data subsets, in which we can show that the expected staleness can be independently estimated in each replica. Suppose the expected staleness and its drift be:
\begin{equation}
 \rmD^t \coloneqq \E\left[\hat{\rvw}_i^t - \hat{\rvw}_i^{t-\tau}\right]\ ,\qquad  \bfalpha_t \coloneqq \rmD^{t} - \rmD^{t-1}\ .  
\end{equation}
With standard assumptions from stochastic approximation theory~\cite{robbins1951stochastic,robbins1971convergence} such as diminishing \gls{EMA} coefficient, and diminishing drift in expected staleness, we show that asynchronous (\ie, delayed) averaging ensures consensus on expectation. This, together with \thmref{thm:sparta}, provides a theoretical justification for asynchronous sparse averaging.
\begin{thm}[Delayed averaging with \gls{EMA} ensures consensus]\label{thm:ema}
    Consider a homogeneous setup where the average staleness $\rmD^t$ is bounded and its drift is diminishing, \ie, $\lim_{t\to\infty}\norm{\bfalpha_t} = 0$. Then, the \gls{EMA} based delay correction as in~\eqref{eq:ema} with $\lambda_t$ satisfying 
    $\sum_t \lambda_t = \infty$, $\sum_t \lambda_t^2 < \infty$, and $\sum_t \frac{\norm{\bfalpha_t}^2}{\lambda_t} < \infty$ ensures consensus on expectation, \ie, $\lim_{t\to \infty}\E\left[\norm{\bfdelta^{t}}^2 \right] = 0$.
\end{thm}
\vspace{-2ex}
\begin{proof}
    Intuitively, in a homogeneous setup, a particular weight trajectory is equally probable for all replicas, and hence, the expected staleness can be independently estimated in each replica. Moreover, classical stochastic approximation theory~\cite{robbins1951stochastic} shows that \gls{EMA} approximates the expected value. Here, since $\rmD^t$ is time varying, we use the additional assumption on the interplay between the drift $\bfalpha_t$ and $\lambda_t$ to bound the consensus error. Detailed proof is provided in the appendix.
\end{proof}
\vspace{-2ex}
The assumption of diminishing $\norm{\bfalpha_t}$ and $\sum_t \frac{\norm{\bfalpha_t}^2}{\lambda_t} < \infty$ can be interpreted as the optimization trajectory being smoother such that the average staleness changes slowly with time and its difference diminishes towards convergence. This also implicitly puts a restriction on the allowed delay $\tau$. To this end, sparse averaging with asynchronous updates is appealing as we can reduce the delay $\tau$ by only communicating a small subset over a bandwidth limited interconnect.

Note the theory recommends diminishing \gls{EMA} coefficient $\lambda_t$. In our implementation, we initialize $\lambda_t$ to $0.5$ and after 1k iterations, we gradually decay it to $0.01$ using a cosine schedule. This slightly improves over a constant $\lambda_t$, and we keep this schedule fixed for all the experiments. In the appendix, we have practically verified that for our method, the EMA estimate variance vanishes and the consensus error approaches zero, as predicted by our theory.

\section{Related Work}
\paragraph{Communication Efficient \gls{DP} Methods.}
\gls{DP} is a traditional distributed training setup~\cite{goyal2017accurate,li2020pytorch}, where each device computes the gradients of the full model in its data split, aggregates the gradients on a central server, performs a central optimization step, and distributes the updated model parameters back to the devices for the next iteration. To reduce the communication overhead, the gradients can be compressed using low-rank approximation~\cite{vogels2019powersgd,zhao2024galore}, sparsification~\cite{wang2017efficient,wangni2018gradient}, and quantization~\cite{bernstein2018signsgd,wang2023cocktailsgd}. Alternatively, in the serverless setup, each device individually updates the model and periodically synchronizes the model parameters, where partial averaging~\cite{sparta,fournier2024wash,lee2023partial} and/or infrequent communication as in \diloco{}~\cite{douillard2025streaming,douillard2023diloco,reddi2020fedopt} reduce communication overhead.

\paragraph{Asynchronous \gls{DP} Methods.} 
Unlike synchronous methods, asynchronous \gls{DP} methods can fully eliminate the communication overhead. These are well-studied for the parameter server setup that communicates the gradients, and many gradient delay correction mechanisms have been developed~\cite{agarwal2011distributed,dp-async-survey,stich2019error}. Notable methods that improve over the simple asynchronous SGD~\cite{recht2011hogwild} include delay dependent learning rate~\cite{barkai2019gap,mishchenko2206asynchronous}, gradient forecasting with second-order information~\cite{zheng2017asynchronous}, and look-ahead in the weight space~\cite{dana}. Apart from this, training dynamics of such asynchronous \gls{DP} methods have also been analyzed~\cite{asynchrony-begets-momentum}. These methods are not directly applicable in our setup, as we communicate weights instead of gradients between replicas. Asynchronous methods are underexplored in this context, although there are a few recent empirical methods designed for \diloco{}~\cite{ajanthan2025momentum,kale2025eager,liu2024asynchronous}, we show some of them~\cite{kale2025eager} can be viewed as special cases of our method.

\begin{figure*}[t]
    \centering
    \begin{subfigure}{0.33\linewidth}
    \includegraphics[width=0.99\linewidth]{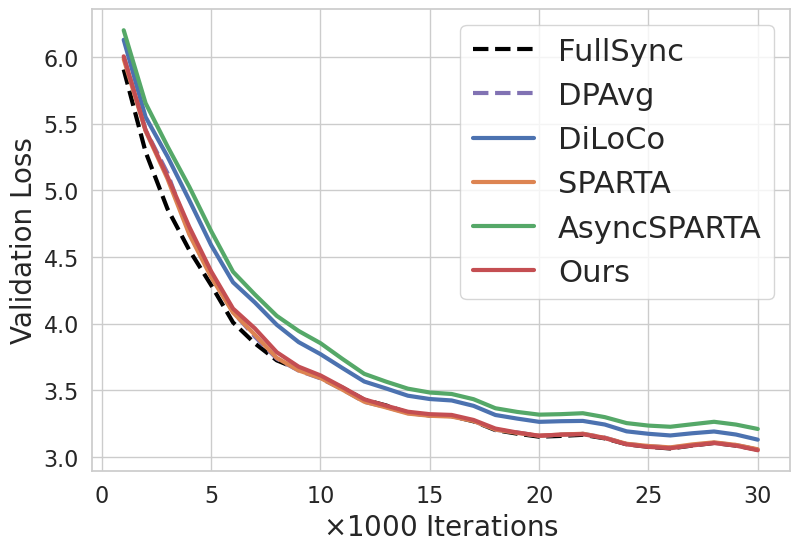}
    \vspace{-4ex}
    \caption{$4\times 2$ mesh}
    \end{subfigure}%
    \begin{subfigure}{0.33\linewidth}
    \centering
    \includegraphics[width=0.99\linewidth]{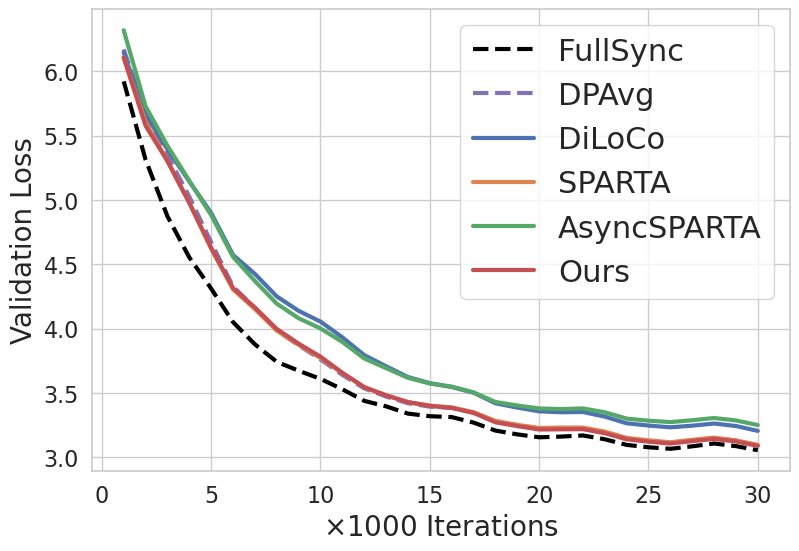}
    \vspace{-4ex}
    \caption{$8\times 2$ mesh}
    \end{subfigure}%
    \begin{subfigure}{0.33\linewidth}
    \includegraphics[width=0.99\linewidth]{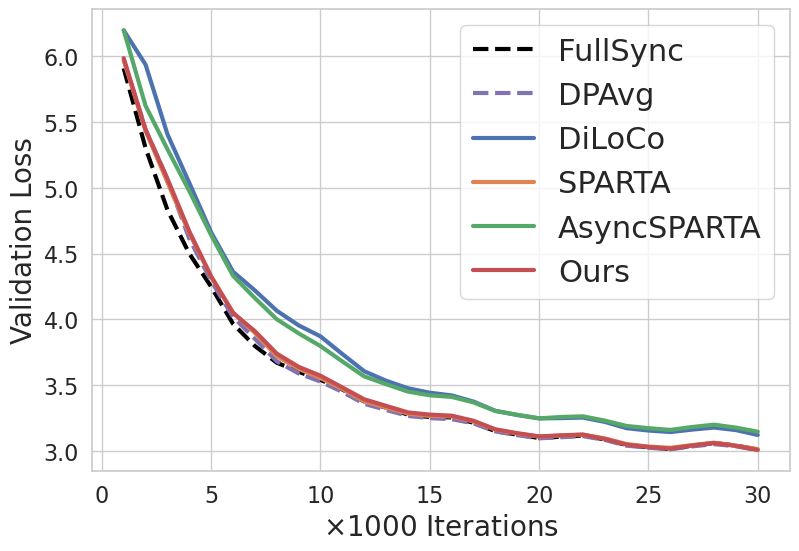}
    \vspace{-4ex}
    \caption{$4\times 2$ mesh, \moe{}}
    \end{subfigure}
    \vspace{-1ex}
    \caption{\em Results on \wiki{} with varying mesh configurations with \asyncpp{} for all methods except FullSync. In all scenarios, our method matches the performance of the fully synchronous method, while outperforming the fully asynchronous baseline \asyncsparta{}.
    }
    \vspace{-2ex}
    \label{fig:main}
\end{figure*}

\paragraph{Asynchronous \gls{PP} Methods.}
Asynchronous \gls{PP} methods eliminate the synchronization bottleneck in \gls{PP}~\cite{pp-survey} to achieve 100\% pipeline utilization at the cost of gradient staleness. In \gls{PP}, weight stashing is used to ensure correct backpropagation is performed~\cite{narayanan2019pipedream,pipedream-2bw}, however, the discrepancy (\ie, staleness) between weights and gradients still needs to be compensated. For this, many correction mechanisms such as learning rate discounting~\cite{yang2021pipemare}, direct weight prediction~\cite{spec-train,guan2019xpipe}, and extrapolation in the weight space~\cite{ajanthan2025asyncpp,zuo2025exploring} have been developed.

In this work, for the first time, we consider \asyncmesh{}, where both \gls{DP} and \gls{PP} are asynchronous. For \gls{PP}, we adopt the recent weight extrapolation method~\cite{ajanthan2025asyncpp}, and for \gls{DP}, we combine sparse averaging~\cite{sparta} with asynchronous updates to fully eliminate the \gls{DP} communication overhead.

\section{Experiments}
\subsection{Experimental Setup} 
We evaluate on four large-scale language modelling datasets, namely, \acrfull{WT}~\cite{wikitext}, \acrfull{BC}~\cite{bookcorpus}, \acrfull{OWT}~\cite{owt}, and \acrfull{FW}~\cite{penedo2024fineweb}, using decoder-only architectures with varying mesh configurations, denoted by {\tt PP-stages $\times$ DP-replicas}. 
Our architecture is based on NanoGPT~\cite{Karpathy2022} with no dropout, and we also test a \gls{MoE} version by replacing every second transformer block with the \moe{} layer. The base model has a context length of 1024, an embedding dimension of 768, 12 attention heads, and 12 layers, with approximately 163M parameters. 
We use the GPT2 tokenizer~\cite{radford2019language} and train from scratch.
For Async\gls{PP}, NAdamW optimizer~\cite{nadam} with momentum $0.99$ is used as per~\cite{ajanthan2025asyncpp}, and all other hyperparameters are set to the standard values (refer to appendix), and kept constant for all the experiments. 

{\em Our aim is to show that neither asynchronous updates for both \gls{PP} and \gls{DP}, nor the sparse averaging for \gls{DP}, deteriorate the validation performance, in various configurations.} For fair comparison, we compute the validation loss (and perplexity) on the model averaged across all \gls{DP} replicas (\ie, consensus model), for all the methods. Primarily, we compare our method against {\bf FullSync}, which is the ideal synchronized setup {\em without any communication efficient scheduling} for \gls{PP} and only utilizes $\frac{1}{P}$ fraction of the pipeline compute~\cite{huang2019gpipe,yang2021pipemare}.
In addition, we evaluate the following synchronous \gls{DP} methods: \DP{} that averages all the parameters at every step, \gls{SPARTA}~\cite{sparta} that averages a small subset, \diloco{}~\cite{douillard2023diloco} that performs infrequent communication, and \asyncsparta{} that performs asynchronous sparse averaging without delay correction. For sparse averaging methods, the subset size is 5\%, the averaging interval is 1, and the asynchronous methods incur a 10-step delay, unless specified otherwise. 

Our method is implemented in PyTorch~\cite{paszke2019pytorch} using the publicly available codes of \gls{SPARTA}\footnote{\url{https://github.com/matttreed/diloco-sim}}, and \asyncpp{}\footnote{\url{https://github.com/PluralisResearch/AsyncPP}}. Unless otherwise specified, all experiments use the base architecture described above and are performed on the \wiki{} dataset. All our experiments are performed on a system equipped with 8 A100 GPUs (p4d.24 in AWS), and where needed, multiple such instances were used.

        
    

\subsection{Main Results}
We analyze the validation loss trajectories for different mesh configurations and architectures with \asyncpp{} in \figref{fig:main}. In all scenarios, our method outperforms \asyncsparta{} and matches the performance of FullSync.  As shown in the appendix, the behaviour is similar for synchronous \gls{PP} updates as well. Except \asyncsparta{} and \diloco{}, all methods show near identical performance in most cases. Note \diloco{} is a prominent method in the \gls{DP} setup with full model in each replica, however, it seems to be inferior in the mesh setup with \asyncpp{}. To our knowledge, this is the first time \diloco{} is tested with \asyncpp{}.
%
\leavevmode\newline
\begin{wrapfigure}{r}{0.5\textwidth}
\vspace{-0.5ex}
    \centering
    \small    
    \begin{tabular}{l|cccc}
        \toprule
        Method   & \wikis{} & \books{} & \owts{} & \fws{} \\
        \midrule
        FullSync & 21.23 & 35.99 & \textbf{35.71} & \textbf{36.77} \\
        \DP{} & 21.26 & 35.24 & 35.97& 36.81 \\
        \midrule
        \acrshort{SPARTA} & 21.30 & 35.15 & 35.73& 37.10\\
        \asyncsparta{} & 24.80 & 37.83 & 41.41& 43.20 \\
        
        \midrule
        Ours & \textbf{21.14} & \textbf{35.09} & 36.13 & 37.31\\
    
        \bottomrule
    \end{tabular}
    \caption{\em Validation perplexity scores at 30k iterations for the $4 \times 2$ mesh. Our method outperforms \asyncsparta{}, and matches FullSync while eliminating the \gls{DP} communication overhead. Except FullSync, all other methods use \asyncpp{}.
    }
    \label{tab:main}
    \vspace{-2ex}
\end{wrapfigure}
In \tabref{tab:main}, we report the validation perplexities on multiple datasets after 30k iterations for the $4\times 2$ mesh. Our method outperforms \asyncsparta{} by {\bf 3 -- 6} points, and yields similar perplexities as FullSync, even surpassing it in 2 out of 4 datasets. 
This is remarkable as our method is fully asynchronous and only averages 5\% of the parameters at each iteration. Note these results are with $2$ \gls{DP} replicas, and as discussed in \secref{sec:ema} and shown in \figref{fig:gap}, {\em the more replicas, the better for our method}. 
For completeness, we trained to the compute-optimal point~\cite{hoffmann2022training} on \fw{}, where the perplexities are 19.92 for FullSync, and 20.10 for ours, confirming the merits of our method beyond doubt.
These results demonstrate that our delay correction method effectively compensates for staleness. 


\begin{figure*}[t]
    \centering
    \includegraphics[width=0.99\linewidth]{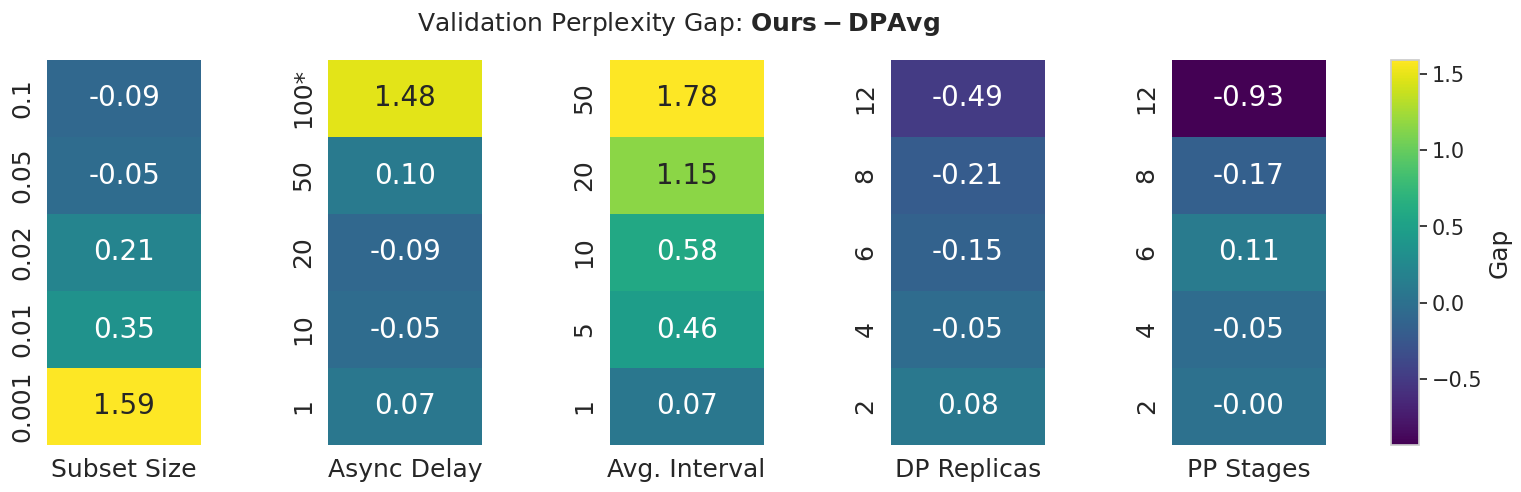}
    \vspace{-1ex}
    \caption{\em Perplexity gap (lower the better) between our method and \DP{} on \wiki{} for different configurations. Our method is robust to a range of configurations and scales favourably with the number of \gls{DP} replicas and \gls{PP} stages. Here, 100* indicates, {\tt async-delay = 10} with {\tt avg-interval = 10}, simulating 100-step effective delay, as {\tt async-delay = 100} did not converge.
    }
    \vspace{-2ex}
    \label{fig:gap}
\end{figure*}

\subsection{Varying Configurations}
\vspace{-1ex}
To test the robustness of our method in a variety of setups, we consider the base model in the following default setup: {\tt subset-size = 5\%}, {\tt async-delay = 10}, {\tt avg-interval = 1}, {\tt DP-replicas = 4}, {\tt PP-stages = 4}, and change one criterion at a time. To isolate the effect of our asynchronous sparse averaging, we compare the validation perplexity against \DP{}, which synchronously averages all parameters, while using \asyncpp{} for pipeline optimization in both methods.

As reported in \figref{fig:gap}, our method is robust for a wide range of configurations, and it is consistently better than \DP{} in many scenarios. Notably, our method is robust for subset size as low as 1\% (only a fractional change in perplexity), tolerates delay up to 50 steps, and works reasonably even if we average every 10 steps (\ie, $10\times$ less communication). More encouraging results are that {\em our method becomes better with a larger mesh} -- with a larger number of \gls{DP} replicas or a larger number of \gls{PP} stages, the Gap: {Ours - \DP{}} becomes more negative. This shows the superior scalability of our method compared to full synchronous averaging.

Note that our theory predicts (refer to \secref{sec:ema}) that a larger number of \gls{DP} replicas would yield a better approximation of average staleness, leading to better delay correction. These results seem to corroborate that. Meanwhile, \asyncpp{} is shown to degrade with more \gls{PP} stages~\cite{ajanthan2025asyncpp}, however, the gains from our sparse averaging may help offset this.

\vspace{-1.5ex}
\paragraph{Practical Implications.}
Compared to \gls{SPARTA}, our method achieves $\mathbf{1.5\times}\text{ -- }\mathbf{3.7\times}$ {\bf speed-up}\footnote{This speed-up is achieved on p4 instances, which have datacenter grade interconnects, and in decentralized setups the speed-up is proportionally higher with lower interconnect bandwidth.}, solely due to asynchronous sparse averaging, as per our runtime measurements. Here, with larger meshes and model sizes, yield better speed-ups (refer to appendix for detailed results and discussion).
Moreover, the above results show that our method tolerates a 50-step delay with only a 5\% subset size, reducing communication by $10\times$ per iteration (subset and indices need to be exchanged). For simplicity, assuming 1s per forward-backward pass in a stage, this gives 50s for \gls{DP} communication -- enough to support up to 1.5B parameters (FP32) per stage over a 100 Mbps connection, highlighting the viability of decentralized training over the internet.


\begin{figure}
    \centering
    \begin{subfigure}{0.45\linewidth}
    \includegraphics[width=0.99\linewidth]{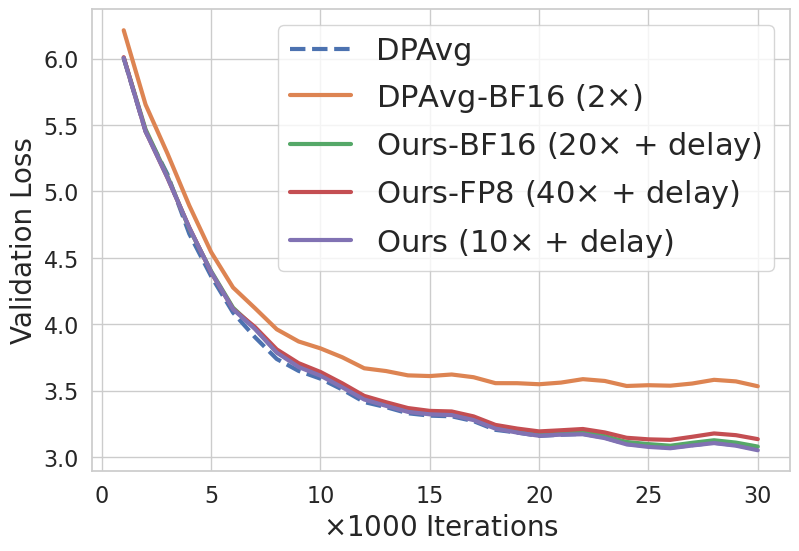}
    \end{subfigure}%
    \begin{subfigure}{0.45\linewidth}
    \includegraphics[width=0.99\linewidth]{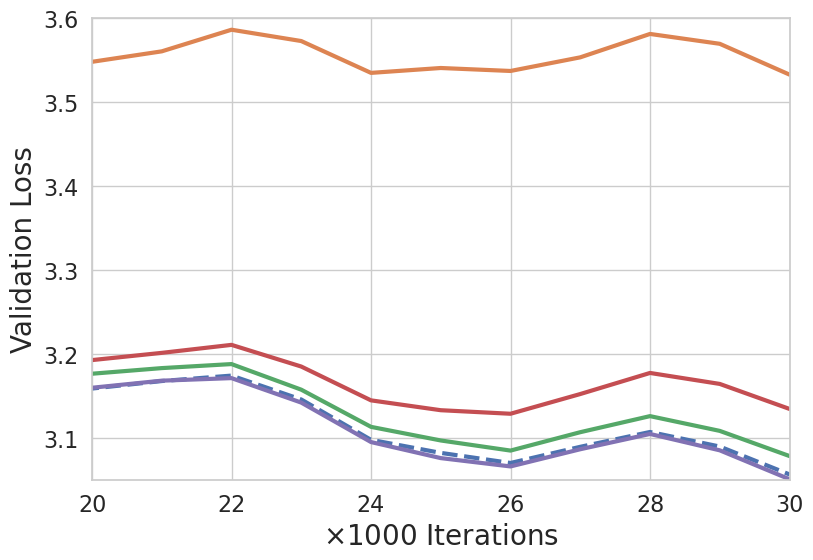}
    \end{subfigure}
    \vspace{-2ex}
    \caption{\em Effect of quantization for different methods, for the base model with $4\times 2$ mesh on \wiki{}. Quantization degrades the performance for \DP{} (\DP{}-{\em FP8} did not converge) while sparse averaging (even with 10-step delay) is robust to it. While intriguing, it may be explained by the fact that since the quantization error is introduced only for a small subset (5\% in this case) at each iteration, the effect of quantization on training is negligible.
    }
    \label{fig:quant}
\end{figure}

\vspace{-1.5ex}
\paragraph{Robustness of Sparse Averaging.}
An intriguing observation is that, our method diverges with subset sizes $\ge\!40\%$ and performs similarly for the subset size range $5\%-30\%$. While, this is counterintuitive, it may indicate the robustness properties of sparse averaging to temporary weight inconsistencies (due to only a small subset being averaged, the overall error introduced to the averaged model is small), also noted in the original paper~\cite{sparta}. To further corroborate this hypothesis, we applied quantization to our method and \DP{}; As shown in \figref{fig:quant}, while our method is robust up to FP8 quantization, \DP{} showed degradation even with BF16 and diverged for FP8. This not only makes our method $\bf{40\times}$ {\bf communication efficient} than \DP{}, but also makes sparse averaging a necessary component of AsyncMesh. Refer to appendix for detailed results.
%
%
\leavevmode\newline
\begin{wrapfigure}{r}{0.5\textwidth}
\vspace{-3ex}
    \centering
    \includegraphics[width=0.85\linewidth]{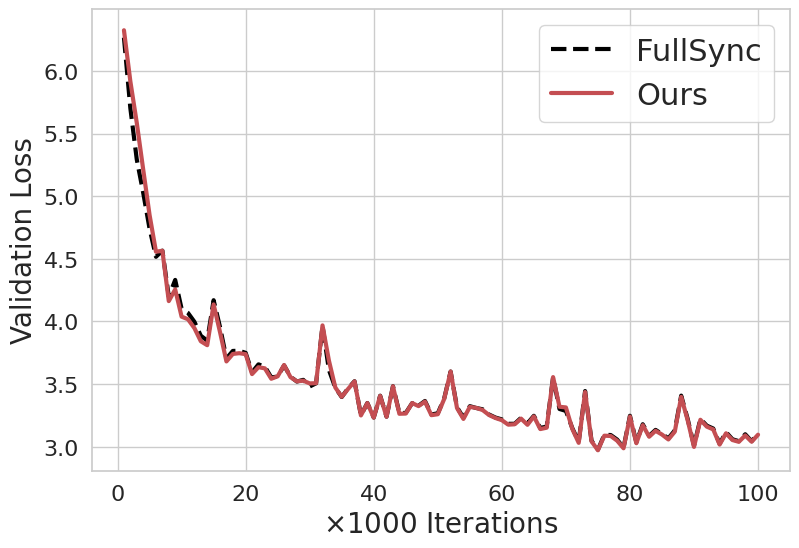}
    \vspace{-2ex}
    \caption{\em Results on the 1B parameter model on \fw{} for $4\times 2$ mesh. Our method matches FullSync, similar to the base model.
    }
    \vspace{-2ex}
    \label{fig:1b}
\end{wrapfigure}
\vspace{-6ex}
\subsection{Increasing the Model Size}
\vspace{-1ex}
To demonstrate the scalability of our approach, we train a {\em 1B parameter model} in the asynchronous 2D mesh. We maintain the number of stages at 4, but increase the embedding dimension to 2304, with 24 attention heads. 
As illustrated in~\figref{fig:1b}, the results align with those of the base model. Specifically, our approach matches FullSync throughout training, yielding validation perplexity of \textbf{19.53} compared to 19.62 for FullSync. This experiment demonstrates the merits of our method and the feasibility of \asyncmesh{} optimization for large-scale language model training.

\subsection{Heterogeneous Setup}
So far, our experiments have been on a homogeneous setup where the compute capability of the devices is the same. To stress test our method, we simulate a heterogeneous setup for a $4\time 4$ mesh by varying device speeds (refer to appendix for details). To ensure each device initiates the averaging operation (\ie, all-reduce) at approximately the same time, we set the averaging interval proportional to the device speed, \ie, faster devices will perform more iterations between an averaging step. This means, model replicas at different training steps are sparsely averaged with a delay due to asynchronous \gls{DP}. This is analogous to dynamic local updates in~\cite{liu2024asynchronous}.  For fair comparison, we fix the total number of iterations across all replicas to be 120k.

As shown in \figref{fig:het}, even with up to $10\times$ difference in device speeds, and delayed sparse averaging of replicas at different training steps, the degradation of validation loss remains small. This degradation is negligible when accounting for the wall-clock speed-up, which scales with device speeds due to the absence of \gls{DP} communication overhead.

\begin{figure}
    \centering
    \begin{subfigure}{0.45\linewidth}
    \includegraphics[width=0.99\linewidth]{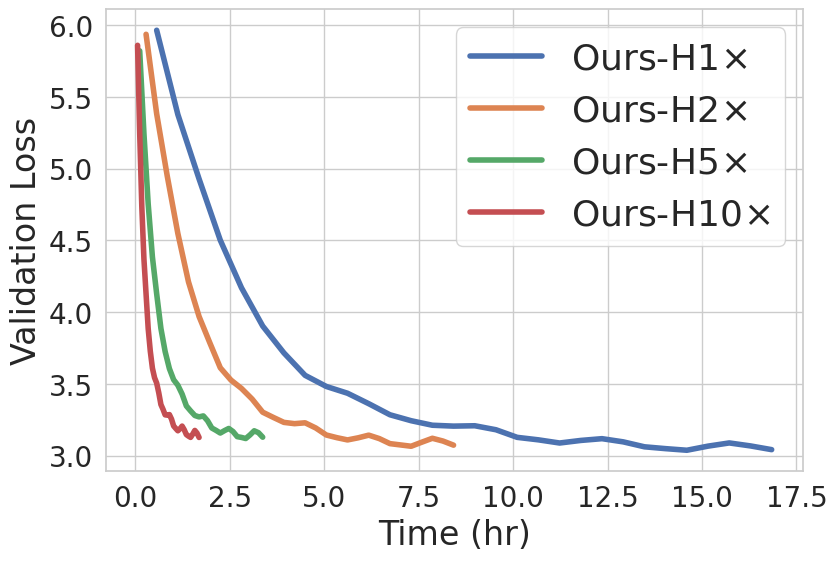}
    \end{subfigure}%
    \begin{subfigure}{0.45\linewidth}
    \includegraphics[width=0.99\linewidth]{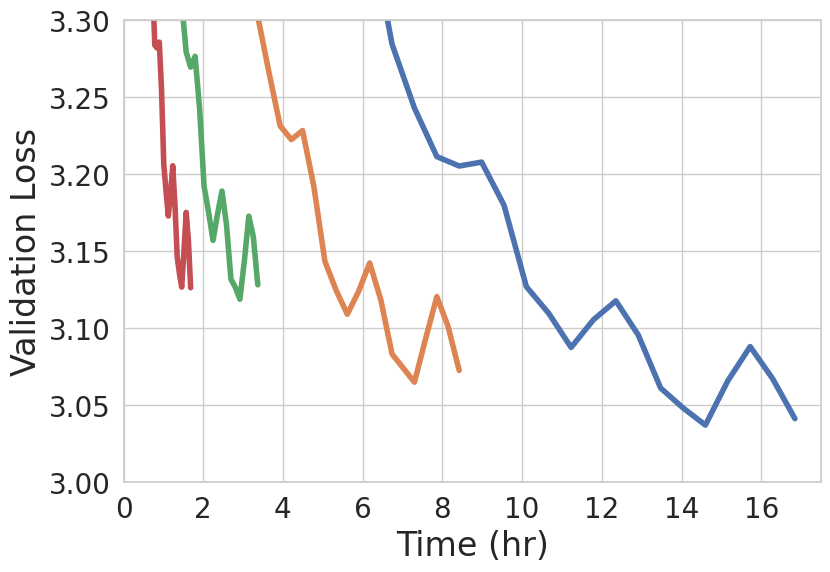}
    \end{subfigure}
    \vspace{-2ex}
    \caption{\em Our method in a heterogeneous setup with drastically varying device speeds. The performance degradation is negligible compared to the gain in wall-clock time.
    }
    \label{fig:het}
\end{figure}

\begin{figure}
    \centering
    \begin{subfigure}{0.5\linewidth}
    \includegraphics[height=0.64\linewidth]{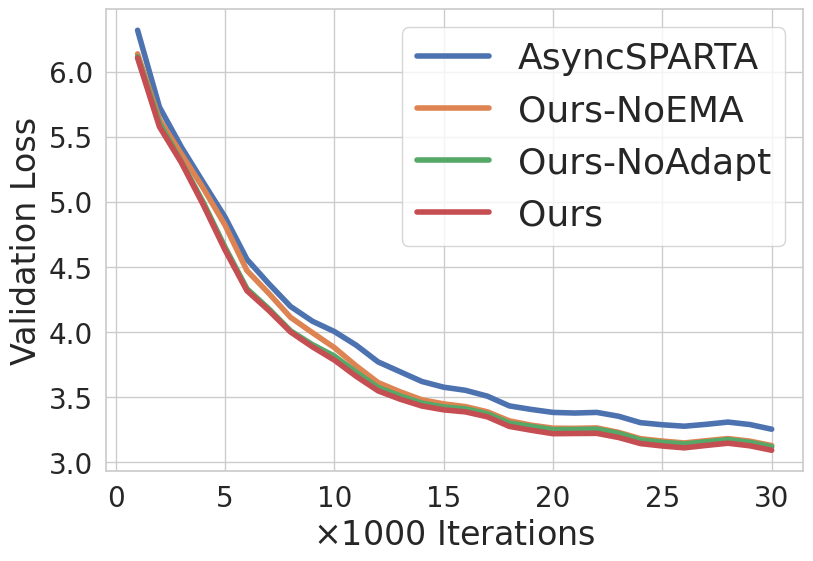}
    \end{subfigure}%
    \begin{subfigure}{0.25\linewidth}
    \includegraphics[height=1.28\linewidth]{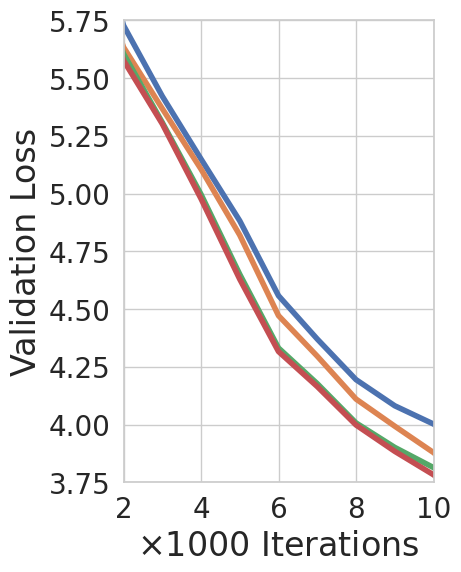}
    \end{subfigure}%
    \begin{subfigure}{0.25\linewidth}
    \includegraphics[height=1.28\linewidth]{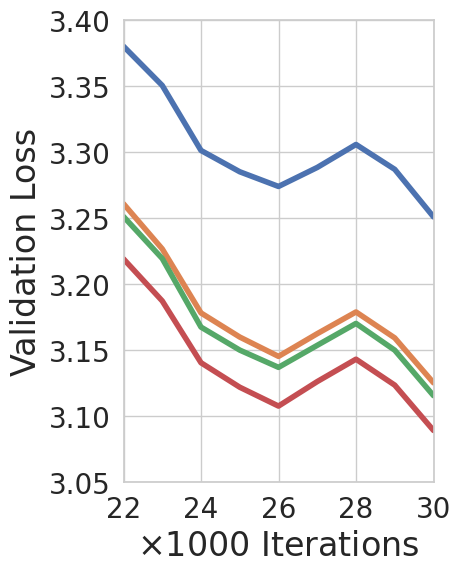}
    \end{subfigure}
    \vspace{-4ex}
    \caption{\em Ablation of our method on \wiki{} for the $8\times 2$ configuration. \gls{EMA} improves the early phase of training, and the adaptive momentum coefficient yields further marginal improvement. 
    }
    \vspace{-2ex}
    \label{fig:abl}
\end{figure}

\subsection{Ablation Study}
To understand the effect of different components of our method, we perform an ablation study in \figref{fig:abl} for the $8\times 2$ mesh on \wiki{}. Simply using weight differences within each replica, \ie, $\rvd_i^t = \hat{\rvw}_{i}^t - \hat{\rvw}_{i}^{t-\tau}$, substantially improve \asyncsparta{}, where \gls{EMA} and adaptive momentum coefficient as predicted by our theory further improves. Note that the effect of \gls{EMA} is more prominent in the early phase of training, where the step sizes are larger (\ie, noisier). This aligns with our intuition that \gls{EMA} robustly approximates the average staleness.


\section{Conclusion}
In this paper, we studied a new \asyncmesh{} setup where both \gls{DP} and \gls{PP} are asynchronous, and introduced a fully asynchronous approach that can match the performance of the fully synchronized method. We theoretically showed that our method converges to a fixed point of the consensus objective on expectation, despite sparse averaging and asynchronous communication. Our experiments on a wide range of configurations demonstrate the merits of our approach, and show the feasibility of asynchronous optimization for large scale language model training. By alleviating communication overhead without any performance penalty, our approach takes a step towards realizing large-scale collaborative training over the internet.

%% file: appendix.tex
\section{Theoretical Insights}
We restate the problem setup, introduce some notations, and then turn to the proofs.
\subsection{Setup and Notations}
We consider the \gls{DP} setup without \gls{PP} for simplified theoretical analysis. In this, the consensus objective takes the following form:
\begin{align}\label{eq:consensus1}
    \min_{\rvw\in\R^d} f(\rvw;\gD) \coloneqq &\min_{\rvw_i\in\R^d}\, \sum_{i=1}^m f(\rvw_i; \gD_i)\ ,\\\nonumber
    &\text{s.t.}\quad \rvw_i = \rvw\ ,\qquad\forall\,i\in \{1,\ldots,m\}\ .
\end{align}
Here, $i\in\{1,\ldots,m\}$ denotes the workers, $\rvw_i, \gD_i$ denote the model weights and data chunk for worker $i$, and $f$ is the loss function. Typically, $\gD_i$ is an \iid subset of $\gD$.

Let us first define some notations:
\begin{align}\label{eq:notations}
\bar{\rvw}^t &\coloneqq \frac{1}{m}\sum_{i=1}^m \rvw_i^t\ ,\qquad&\mbox{averaged weights}\ ,\\\nonumber
\rvg_i^t &\coloneqq \nabla f_i(\rvw_i^t; \gB_i^t)\ ,\qquad&\mbox{minbatch gradient}\ ,\\\nonumber
\bar{\rvg}^t &\coloneqq \frac{1}{m}\sum_{i=1}^m \rvg_i^t\ ,\qquad&\mbox{average of gradients}\ ,\\\nonumber
\norm{\bfdelta^t}^2 &\coloneqq \sum_{i=1}^m \norm{\rvw_i^t - \bar{\rvw}^t}^2\ ,\qquad&\mbox{consensus error}\ .\\\nonumber
\end{align}
Here, $\norm{\bfdelta^t}^2$ denotes the post-averaging consensus error at iteration $t\!-\!1$, and the corresponding pre-averaging error is denoted as $\norm{\hat{\bfdelta}^t}^2$. Analogously, $\bfdelta^t_i \coloneqq \rvw_i^t - \bar{\rvw}^t$.

We are now ready to prove that sparse averaging ensures consensus across the weights of all workers on expectation.

\subsection{Sparse Averaging Ensures Consensus}

\begin{lem}\label{lem:sparta1}
    Let $p$ be the probability of an element $\mu\in\gS^t$ independent of others, then sparse averaging shrinks the consensus error by a factor of $(1\!-\!p)$ on expectation, \ie, ${\E\left[\norm{\bfdelta^{t}}^2\right] = (1-p)\,\E\left[\norm{\hat{\bfdelta}^{t}}^2\right]}$.
\end{lem}
\begin{proof}
    By definition of sparse averaging, $w_{i:\mu}^{t+1} = \frac{1}{m}\sum_i\hat{w}_{i:\mu}^{t}$ for all $\mu\in\gS^t$. Therefore, 
    \begin{align}
        \E\left[\norm{\bfdelta^{t}}^2\right] &= \sum_{i=1}^m\sum_{\mu=1}^d \E\left[\left(\bfdelta_{i:\mu}^{t}\right)^2\right]\ ,\qquad&\mbox{linearity of expectation}\ ,\\\nonumber
        &= \sum_{i=1}^m\sum_{\mu=1}^d \left(\hat{\bfdelta}_{i:\mu}^{t}\right)^2\,\E\,\id{\mu\notin \gS^t}\ ,\qquad&\mbox{$\id{\cdot}$ is the indicator}\ ,\\\nonumber
        &= (1-p)\,\E\left[\norm{\hat{\bfdelta}^{t}}^2\right]\ .\qquad&\mbox{$P\id{\mu\notin \gS^t} = 1-p$}\ .
    \end{align}
\end{proof}

\begin{thm}\label{thm:sparta1}
    Let $f$ be a $L$-smooth function, the stochastic gradient $\rvg_i^t$ be an unbiased estimate of $\nabla f$ and have bounded variance $\sigma^2$, and $p>0$ be the averaging probability for an element $\mu$, then, for an appropriate choice of learning rate $\eta_t > 0$ the consensus error diminishes on expectation, \ie, $\lim_{t\to \infty}\E\left[\norm{\bfdelta^{t}}^2 \right] = 0$.
\end{thm}
\begin{proof}
    Let us expand $\hat{\bfdelta}^{t+1}_i$:
    \begin{align}
        \hat{\bfdelta}^{t+1}_i &= \hat{\rvw}_i^{t+1} - \frac{1}{m}\sum_i\hat{\rvw}_i^{t+1}\ ,\\\nonumber
        &= \rvw_i^t - \eta_t\,\rvg_i^t - \frac{1}{m}\sum_i\left(\rvw_i^t - \eta_t\,\rvg_i^t\right)\ ,\qquad&\mbox{local update}\ ,\\\nonumber
        &= \bfdelta_i^{t} - \eta_t\left(\rvg_i^t - \bar{\rvg}^t\right)\ .
    \end{align}
    Now, the pre-averaging consensus error can be written as:
    \begin{align}
        \norm{\hat{\bfdelta}^{t+1}}^2 &= \sum_i \norm{\bfdelta_i^{t} - \eta_t\left(\rvg_i^t - \bar{\rvg}^t\right)}^2\ ,\\\nonumber
        &= \sum_i \left(\norm{\bfdelta_i^{t}}^2 - 2\,\eta_t\,\bfdelta_i^t \cdot \left(\rvg_i^t - \bar{\rvg}^t\right) + \eta_t^2 \norm{\rvg_i^t - \bar{\rvg}^t}^2\right)\ ,\qquad&\mbox{$\cdot$ is the dot product}\ ,\\\nonumber
        &= \sum_i \left(\norm{\bfdelta_i^{t}}^2 - 2\,\eta_t\,\bfdelta_i^t \cdot \rvg_i^t + \eta_t^2 \norm{\rvg_i^t - \bar{\rvg}^t}^2\right)\ ,\qquad&\mbox{$\sum_i \bfdelta_i^t = 0$}\ .
    \end{align}
    We now bound each of the terms following techniques similar to that of~\cite{yu2019linear-proof}.
    Consider $ \norm{\rvg_i^t - \bar{\rvg}^t}^2$:
    \begin{align}
         \norm{\rvg_i^t - \bar{\rvg}^t}^2 &=  \norm{\rvg_i^t - \nabla f(\rvw_i^t) + \nabla f(\rvw_i^t) - \bar{\nabla} f(\rvw_i^t) + \bar{\nabla} f(\rvw_i^t) - \bar{\rvg}^t}^2\ ,\\\nonumber
         &\le 3\norm{\rvg_i^t - \nabla f(\rvw_i^t)}^2 + 3\norm{\nabla f(\rvw_i^t) - \bar{\nabla} f(\rvw_i^t)}^2 + 3\norm{\bar{\nabla} f(\rvw_i^t) - \bar{\rvg}^t}^2\ ,
    \end{align}
    where $\bar{\nabla} f(\rvw_i^t) = \frac{1}{m}\sum_i \nabla f(\rvw_i^t)$, and the second step is due to $(a+b+c)^2 \le 3(a^2+b^2+c^2)$. Each of these terms can be bounded as follows: 1) using bounded variance,
    \begin{align}
        \sum_i \norm{\rvg_i^t - \nabla f(\rvw_i^t)}^2 &\le m\,\sigma^2 \ .
    \end{align}
    2) using $L$-smoothness and triangle inequality,
    \begin{align}
        \sum_i \norm{\nabla f(\rvw_i^t) - \bar{\nabla} f(\rvw_i^t)}^2 &= \sum_i \norm{\nabla f(\rvw_i^t) - \nabla f(\bar{\rvw}^t) + \nabla f(\bar{\rvw}^t) - \bar{\nabla} f(\rvw_i^t)}^2\ ,\\\nonumber
        &= \sum_i 2\norm{\nabla f(\rvw_i^t) - \nabla f(\bar{\rvw}^t)}^2 + 2\norm{\frac{1}{m}\sum_i \nabla f(\bar{\rvw}^t) - \frac{1}{m}\sum_i\nabla f(\rvw_i^t)}^2\ ,\\\nonumber
        &\le 2 L^2 \sum_i\norm{\rvw_i^t - \bar{\rvw}^t}^2 + \frac{2L^2}{m} \sum_k \sum_i\norm{\rvw_i^t - \bar{\rvw}^t}^2\ ,\\\nonumber
        &= 4L^2 \norm{\bfdelta^{t}}^2\ .
    \end{align}
    3) using triangle inequality and bounded variance, 
    \begin{align}
        \sum_i \norm{\bar{\nabla} f(\rvw_i^t) - \bar{\rvg}^t}^2 &\le \frac{1}{m}\sum_i\norm{\rvg_i^t - \nabla f(\rvw_i^t)}^2 \le \sigma^2\ .
    \end{align}
    Altogether for some constant $C > 0$, we can write:
    \begin{equation}
        \E\left[\sum_i \norm{\rvg_i^t - \bar{\rvg}^t}^2\right] \le C\left(m\,\sigma^2 + L^2 \,\E\left[\norm{\bfdelta^{t}}^2\right]\right)\ .
    \end{equation}
    Now, consider the term $\E\left[\sum_i\bfdelta_i^t \cdot \rvg_i^t\right]$:
    \begin{align}
        \E\left[\sum_i\bfdelta_i^t \cdot \rvg_i^t\right] &= \E\left[\sum_i\bfdelta_i^t \cdot \nabla f(\rvw_i^t)\right]\ ,\qquad&\mbox{$\E\left[\rvg_i^t\mid \bfdelta_i^t\right] = \E\left[\rvg_i^t\mid \rvw_i^t\right] = \nabla f(\rvw_i^t)$}\ ,\\\nonumber
        &= \E\left[\sum_i\bfdelta_i^t \cdot \left(\nabla f(\rvw_i^t)-\nabla f(\bar{\rvw}^t)\right)\right]\ ,\qquad&\mbox{$\sum_i \bfdelta_i^t = 0$}\ ,\\\nonumber
        &\le L\, \E\left[\sum_i\norm{\bfdelta_i^t} \cdot \norm{\bfdelta_i^t}\right]\ ,\qquad&\mbox{$L$-smooth, $\bfdelta^t_i = \rvw_i^t - \bar{\rvw}^t$}\ ,\\\nonumber
        &= L\, \E\left[\norm{\bfdelta^t}^2\right]\ .
    \end{align}
    Putting everything together,
    \begin{align}
         \norm{\hat{\bfdelta}^{t+1}}^2 &\le (1 + 2\,\eta_t\,L)\,\E\left[\norm{\bfdelta^t}^2\right] + \eta_t^2\,C\left(m\,\sigma^2 + L^2 \,\E\left[\norm{\bfdelta^{t}}^2\right]\right)\ .
    \end{align}
    From Lemma~\ref{lem:sparta1}, 
    \begin{align}
        \E\left[\norm{\bfdelta^{t+1}}^2\right] &= (1-p)\,\E\left[\norm{\hat{\bfdelta}^{t+1}}^2\right]\ ,\\\nonumber
        &\le (1-p)\,(1 + 2\,\eta_t\,L)\,\E\left[\norm{\bfdelta^t}^2\right] + \gO(\eta_t^2)\ .
    \end{align}
    Note, $\eta_t$ can be chosen such that the quadratic term is vanishes, \ie, ${\eta_t > 0, \sum_t\eta_t=\infty, \text{and} \sum_t\eta_t^2 < \infty}$~\cite{robbins1951stochastic}, and the coefficient ${(1-p)\,(1 + 2\,\eta_t\,L)}$ is strictly less than 1, \ie, $\eta_t < \frac{p}{2(1-p)L}$. 
    
    This yields a contraction and ensures ${\lim_{t\to \infty}\E\left[\norm{\bfdelta^{t}}^2 \right] = 0}$.
\end{proof}

This proves that sparse averaging can lead to consensus among workers, in the sense that, on expectation, all weight vectors converge to their average. This, together with the standard convergence proof of \gls{SGD}~\cite{bottou2018optimization} guarantees that sparse averaging converges to a fixed point of \eqref{eq:consensus1}.

\subsection{EMA based Delay Correction Ensures Consensus}
We consider a {\em homogeneous setup} where all workers are initialized to the same point, the data chunks are \iid, and the optimizer parameters are identical. In this, we first show that the expected staleness $\E\left[\hat{\rvw}_i^t - \hat{\rvw}_i^{t-\tau}\right]$ can be independently estimated in the worker $i$. Then, we prove that the drift of $\rvd_i^t$ between different workers diminishes, ensuring consensus. In this section, with a slight abuse of notation, we define $\bar{\rvw}^t = \frac{1}{m}\sum_i\hat{\rvw}_{i}^{t}$.

\begin{lem}\label{lem:ema1}
    In a homogeneous setup as defined above, the expected value of the weight drift is independent of the worker, \ie, $\E\left[\hat{\rvw}_i^t - \hat{\rvw}_i^{t-\tau}\right] = \rmD^t$.
\end{lem}
\begin{proof}
    Considering a single local update:
    \begin{align}
        \E\left[\hat{\rvw}_i^{k+1} - \rvw_i^k\right] &= -\E\left[\eta_k\,\rvg_i^k\right] = -\eta_k \nabla f(\rvw_i^k)\ ,\qquad&\mbox{$\rvg_i^k$ is unbiased}\ .
    \end{align}
    This is independent and identical for each worker if $\rvw_i^k = \rvw^k$ for all $i$. This argument can be extended to multiple steps for a homogeneous setup due to the same initialization, \iid data samples, and identical optimizer hyperparameters. Intuitively, one can see that a particular weight trajectory is equally probable for all workers in a homogeneous setup. 
    
    Additionally, the \gls{EMA} update parameters ($\rvd_i^0$ and $\lambda_t$) are also identical across workers.    
    Therefore, the average weight drift $\E\left[\hat{\rvw}_i^t - \hat{\rvw}_i^{t-\tau}\right] = \E\left[\hat{\rvw}_i^t - \rvw_i^{t-\tau}\right] + \E\left[\rvw_i^{t-\tau}-\hat{\rvw}_i^{t-\tau}\right]$ is independent of the worker $i$.
\end{proof}

\begin{thm}\label{thm:ema1}
    Consider a homogeneous setup where the average staleness $\rmD^t$ is bounded and its drift $\bfalpha_t \coloneqq \rmD^{t} - \rmD^{t-1}$ is diminishing, \ie, $\lim_{t\to\infty}\norm{\bfalpha_t} = 0$. Then, the \gls{EMA} based delay correction with $\lambda_t$ satisfying 
    $\sum_t \lambda_t = \infty$, $\sum_t \lambda_t^2 < \infty$, and $\sum_t \frac{\norm{\bfalpha_t}^2}{\lambda_t} < \infty$ ensures consensus, \ie, $\lim_{t\to \infty}\E\left[\norm{\bfdelta^{t}}^2 \right] = 0$.
\end{thm}
\begin{proof}
    Now consider $\bfdelta^{t}_i$:
    \begin{align}
        \bfdelta^{t}_i &= \rvw_i^{t} - \bar{\rvw}_i^{t}\ ,\\\nonumber
        &= \bar{\rvw}^{t-\tau} + \rvd_i^t - \frac{1}{m}\sum_i\left(\bar{\rvw}^{t-\tau} + \rvd_i^t\right)\ ,\qquad&\mbox{\gls{EMA}}\ ,\\\nonumber
        &= \rvd_i^t - \bar{\rvd}_i^t\ .
    \end{align}
    Note that, $\bar{\rvd}_i^t = \frac{1}{m} \sum_i \rvd_i^t$ is the empirical estimate of $\E\left[\rvd_i^t\right]$. Following the arguments from Lemma~\ref{lem:ema1}, we can see that $\E\left[\rvd_i^t\right] = \rmD^t$ for all workers. If $\rmD^t$ is a constant, then the standard stochastic approximation theory~\cite{robbins1951stochastic} yields the desired result. However, since $\rmD^t$ varies with time, we need an additional assumption on the interplay between the drift and the \gls{EMA} coefficient, that $\frac{\norm{\bfalpha_t}^2}{\lambda_t}\to 0$ and carefully bound the consensus error.
    
    Substituting $\bar{\rvd}_i^t=\rmD^t$, together with the \gls{EMA} update:
    \begin{align}
        \bfdelta^{t}_i &= \rvd_i^t - \rmD^t\ ,\\\nonumber
        &= (1-\lambda_t)\,\rvd_i^{t-1} + \lambda_t \left(\hat{\rvw}_i^t - \hat{\rvw}_i^{t-\tau}\right) - \rmD^t \ ,\\\nonumber
        &= (1-\lambda_t)\left(\rvd_i^{t-1} - \rmD^{t-1} + \rmD^{t-1} - \rmD^{t}\right) + \lambda_t \left(\hat{\rvw}_i^t - \hat{\rvw}_i^{t-\tau} - \rmD^t\right)\ ,\\\nonumber
        &= (1-\lambda_t)\,\bfdelta_i^{t-1} - (1-\lambda_t)\,\bfalpha_t + \lambda_t\, \bfxi_i^t\ ,\qquad\qquad\mbox{$\bfxi_i^t = \hat{\rvw}_i^t - \hat{\rvw}_i^{t-\tau} - \rmD^t$}\ .
    \end{align}
    Now square both sides:
    \begin{align}
        \norm{\bfdelta_i^{t}}^2 =\ &(1-\lambda_t)^2\norm{\bfdelta_i^{t-1}}^2 + (1-\lambda_t)^2\norm{\bfalpha_t}^2 + \lambda_t^2\norm{\bfxi_i^t}^2 \\\nonumber 
        &+ 2(1-\lambda_t)\lambda_t\,\bfdelta_i^{t-1}\cdot\bfxi_i^t - 2(1-\lambda_t)^2\bfdelta_i^{t-1}\cdot\bfalpha_t - 2(1-\lambda_t)\lambda_t\,\bfalpha_t\cdot\bfxi_i^t\ ,
    \end{align}
    where $\cdot$ is the dot-product.
    Consider the first cross term: 
    \begin{equation}
     \E\left[\bfdelta_i^{t-1}\cdot\bfxi_i^t\right] = \E\left[\bfdelta_i^{t-1}\cdot\E\left[\bfxi_i^t\mid\bfdelta_i^{t-1}\right]\right] = 0\ ,\qquad\mbox{due to Lemma~\ref{lem:ema1}}\ .
    \end{equation}
    Using Young's inequality for the other cross terms, \ie, $2ab \le \epsilon a^2 + \frac{1}{\epsilon}b^2$ for any $\epsilon > 0$, we can write:
    \begin{align}
        \E\left[\norm{\bfdelta_i^{t}}^2\right] \le\ &(1-\lambda_t)^2\,\E\left[\norm{\bfdelta_i^{t-1}}^2\right] + (1-\lambda_t)^2\norm{\bfalpha_t}^2 + \lambda_t^2\,\E\left[\norm{\bfxi_i^t}^2\right] \\\nonumber 
        &+ (1-\lambda_t)^2\left(\epsilon\,\E\left[\norm{\bfdelta_i^{t-1}}^2\right] + \frac{\norm{\bfalpha_t}^2}{\epsilon}\right) + (1-\lambda_t)\lambda_t\left(\epsilon\,\E\left[\norm{\bfxi_i^{t}}^2\right] + \frac{\norm{\bfalpha_t}^2}{\epsilon}\right)\ .
    \end{align}
    The term $\E\left[\norm{\bfxi_i^{t}}^2\right]$ can be bounded due to bounded delay $\tau$. By setting $\epsilon = \lambda_t$, the above can be simplified as:
    \begin{align}
        \E\left[\norm{\bfdelta_i^{t}}^2\right] &\le (1-\lambda_t)^2(1+\lambda_t)\,\E\left[\norm{\bfdelta_i^{t-1}}^2\right] + \frac{A}{\lambda_t}\norm{\bfalpha_t}^2\\\nonumber &\quad+ B \left(\norm{\bfalpha_t}^2+\lambda_t^2\right)\ ,\qquad\mbox{for some constants $A,B > 0$}\ ,\\\nonumber
        &\le (1-\lambda_t)\,\E\left[\norm{\bfdelta_i^{t-1}}^2\right] + \frac{A}{\lambda_t}\norm{\bfalpha_t}^2+ B \left(\norm{\bfalpha_t}^2+\lambda_t^2\right)\ ,\qquad\mbox{$0 < \lambda_t < 1$}\ .
    \end{align}
    Since $\frac{\norm{\bfalpha_t}^2}{\lambda_t}$, $\norm{\bfalpha_t}^2$, and $\lambda_t^2$ are diminishing, the above can be shown to be a contraction and therefore, $\E\left[\norm{\bfdelta_i^{t}}^2\right]\to 0$ following the classical stochastic approximation theory~\cite{robbins1951stochastic,robbins1971convergence}. Consequently, the consensus error vanishes:
    \begin{equation}
        \E\left[\norm{\bfdelta^{t}}^2 \right] = \E\left[\sum_i \norm{\bfdelta_i^t}^2 \right] = \sum_i\E\left[ \norm{\bfdelta_i^t}^2 \right]\to 0\ .
    \end{equation}
\end{proof}

Note that the \gls{EMA} update and the arguments in the theorem are elementwise, hence, they naturally extend to the \gls{PP} with sparse averaging setup. This, together with Theorem~\ref{thm:sparta1} and the convergence proof of \gls{SGD}, provides a theoretical justification for convergence for our method in the \gls{PP} setup with sparse averaging on expectation, despite a fixed delay. 

\section{Experiments}
\subsection{Experimental Setup}
We evaluate on four large-scale language modeling datasets, namely, \acrfull{WT}~\cite{wikitext}, \acrfull{BC}~\cite{bookcorpus}, \acrfull{OWT}~\cite{owt}, and \acrfull{FW}~\cite{penedo2024fineweb}.
For \wiki{}, we utilize the predefined training and validation splits; for \book{} and \owt{}, we randomly select $10\%$ of the training set as the held-out validation set; and for \fw{}, we use the streaming feature in Huggingface datasets and hold out 10k samples in the stream for validation. 
Our architecture is based on NanoGPT~\cite{Karpathy2022} with no dropout. The base model has a context length of 1024, an embedding dimension of 768, 12 attention heads, and 12 layers, with approximately 163M parameters. 
We use the GPT2 tokenizer~\cite{radford2019language} and train the model from scratch.
For configurations with 2 and 4 pipeline stages, equal number of layers are assigned to each stage, unless specified otherwise. For 8 stages, stages 2 -- 5 are assigned 2 layers, and others have 1 layer each. 

Across all experiments, we maintain a microbatch size of 8 per \gls{DP} replica, a learning rate of $3e\text{-}4$, a weight decay of $0.01$, and gradient clipping norm of $1$. For experiments with asynchronous \gls{PP}, NAdamW optimizer~\cite{nadam} with momentum $0.99$ is used as per~\cite{ajanthan2025asyncpp}. For synchronous \gls{PP} experiments, GPipe~\cite{huang2019gpipe} with  AdamW optimizer~\cite{loshchilov2017decoupled} is used, and the number of microbatches is set to 2. Each experiment is run for 30k iterations, with a linear warmup of 3k iterations starting from a learning rate of $1e\text{-}7$. Then, it is decayed to $3e\text{-}5$ following a cosine decay schedule. For our method, the \gls{EMA} variable $\rvd_i^t$ is initialized to zero.

For \diloco{}, the outer learning rate of 1 performs better (\ie, averaging instead of outer optimization step) in our 2D mesh with \asyncpp{}, and the outer-update interval is set to 10 steps in our experiments.

\paragraph{1B Model.}
We maintain the number of stages at 4 with stage assignment of $[1,3,4,4]$ number of layers, but increase the embedding dimension to 2304, with 24 attention heads. The learning rate warmup step is adjusted to 6k for all methods and run for 100k iterations. All other hyperparameters are the same as the base model.

\paragraph{Varying Configurations.}
When changing a criterion, all other criteria are kept to the default values: {\tt subset-size = 5\%}, {\tt async-delay = 10}, {\tt avg-interval = 1}, {\tt DP-replicas = 4}, {\tt PP-stages = 4}. While varying the averaging interval, the asynchronous delay is set to 1, such that the effective delay is equal to the averaging interval. 

\paragraph{Heterogeneous Setup.}
We use a $4\times 4$ mesh and compare three heterogeneous configurations, namely, H$2\times$: $[2, 1, 2, 3]$, H$5\times$: $[5, 3, 1, 11]$, and H$10\times$: $[10, 8, 1, 21]$, along with the homogeneous setup: H$1\times$: $[1, 1, 1, 1]$, with the provided relative device speeds. 

\subsection{Additional Results}
We provide additional validation loss trajectories for various methods with \asyncpp{} in \figref{fig:main0} and synchronous \gls{PP} in \figref{fig:main1}, on multiple datasets corresponding to the results in the main paper in \figref{fig:datasets}, and for compute optimal training for the base model in \figref{fig:co}. Furthermore, we provide consensus error plots for our method confirming the theory in \figref{fig:error}, validation loss vs time plot for the 1B model in \figref{fig:1b1}.

\paragraph{Gain in Wall-clock Time.}
Since we simulate the asynchronous \gls{DP} setup via buffering, and due to implementation differences between our method and FullSync, their practical wall-clock times are not comparable. However, the theoretical gain in wall-clock time per iteration due to asynchronous sparse averaging is $O(d^2/B)$, where $d$ is the embedding dimension, and $B$ is the bandwidth, as we fully eliminate \gls{DP} overhead. Since the staleness is in update steps, the ratio between allowed time and the data transfer volume is $O(Bd^3\tau/pd^2) = O(Bd\tau/p)$, where $\tau$ is the allowed delay, $p$ is the subset size, and note the compute time is approximately cubic in $d$~\cite{ryabinin2023swarm}. Therefore, our method scales favourably for large models. Note, \asyncpp{}~\cite{ajanthan2025asyncpp} further improves this wall-clock time gain.

For completeness, we show the empirical time savings that can be achieved by our asynchronous sparse averaging in Table~\ref{tab:time}. This clearly shows just by making sparse averaging asynchronous with 1-step delay, we can achieve a speed-up of $1.5\times$ -- $3.7\times$ depending on the mesh configuration, with larger meshes yielding better speed-ups.


\paragraph{Comparison with Other Communication Efficient Methods.}
We compare with some existing methods that compress the \gls{DP} communication via quantization and TopK sample in \figref{fig:quant1}. The quantization results show sparse averaging methods (even with delay) are more robust to quantization, allowing further reduction in DP communication requirements. Finally, we compare with the concurrent work of eager updates for \diloco{}~\cite{kale2025eager} in \figref{fig:eager}, showing that our method strictly generalizes it.

\begin{figure*}[t]
    \centering
    \begin{subfigure}{0.45\linewidth}
    \includegraphics[width=0.99\linewidth]{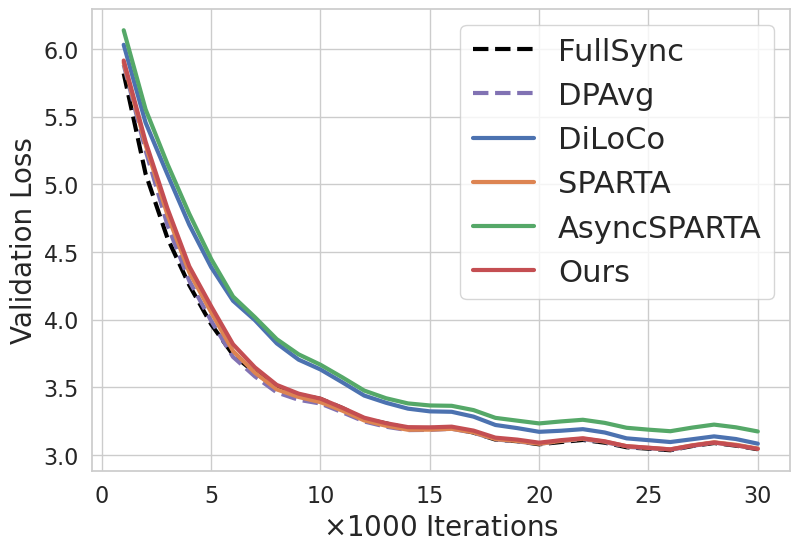}
    \caption{$2\times 4$ mesh}
    \end{subfigure}%
    \begin{subfigure}{0.45\linewidth}
    \includegraphics[width=0.99\linewidth]{images/plot-dp/wiki-pp4-dp2-ppasync-eval_loss.png}
    \caption{$4\times 2$ mesh}
    \end{subfigure}
    \begin{subfigure}{0.45\linewidth}
    \centering
    \includegraphics[width=0.99\linewidth]{images/plot-dp/wiki-pp8-dp2-ppasync-eval_loss.png}
    \caption{$8\times 2$ mesh}
    \end{subfigure}%
    \begin{subfigure}{0.45\linewidth}
    \includegraphics[width=0.99\linewidth]{images/plot-dp/wiki-pp4-dp2-moe-ppasync-eval_loss.png}
    \caption{$4\times 2$ mesh, \moe{}}
    \end{subfigure}
    \caption{\em Results on \wiki{} with varying mesh configurations and architectures with \asyncpp{} for all methods except FullSync. In all scenarios, our method matches the performance of the fully synchronous method, while outperforming the fully asynchronous baseline \asyncsparta{}.
    }
    \label{fig:main0}
\end{figure*}

\begin{figure}[t]
    \centering
    \begin{subfigure}{0.45\linewidth}
    \includegraphics[width=0.99\linewidth]{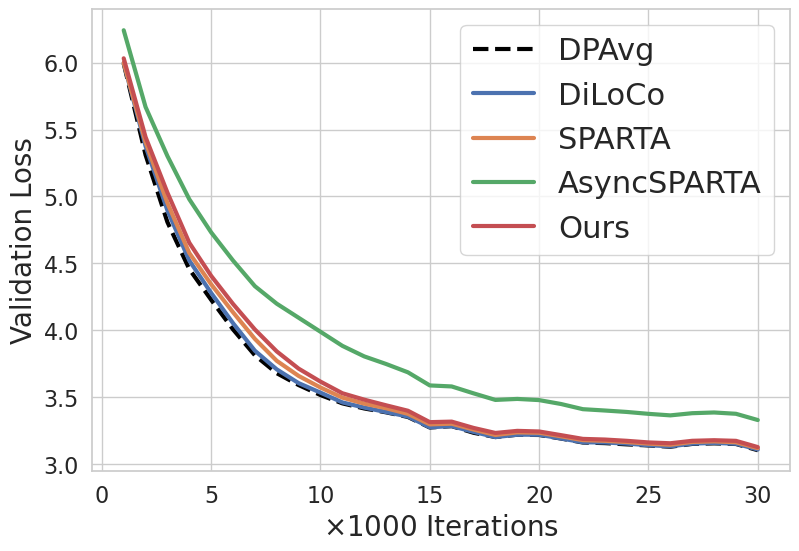}
    \caption{$2\times 4$ mesh}
    \end{subfigure}%
    \begin{subfigure}{0.45\linewidth}
    \includegraphics[width=0.99\linewidth]{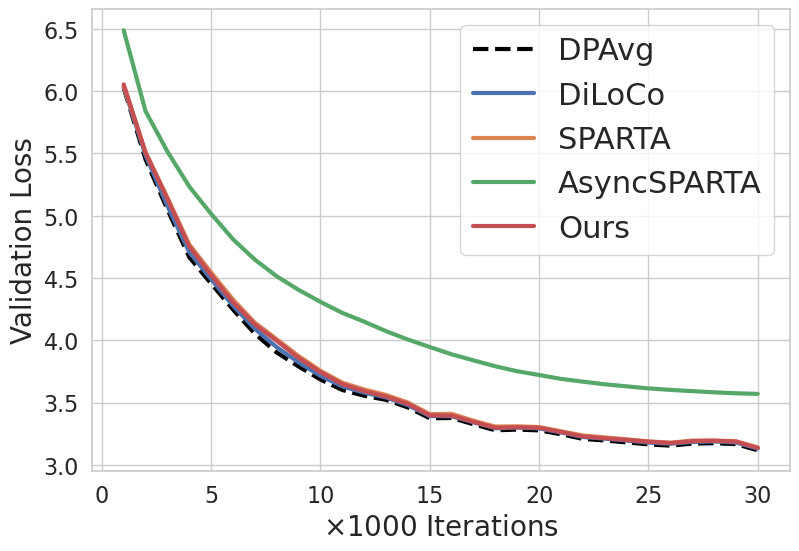}
    \caption{$4\times 2$ mesh}
    \end{subfigure}
    \begin{subfigure}{0.45\linewidth}
    \centering
    \includegraphics[width=0.99\linewidth]{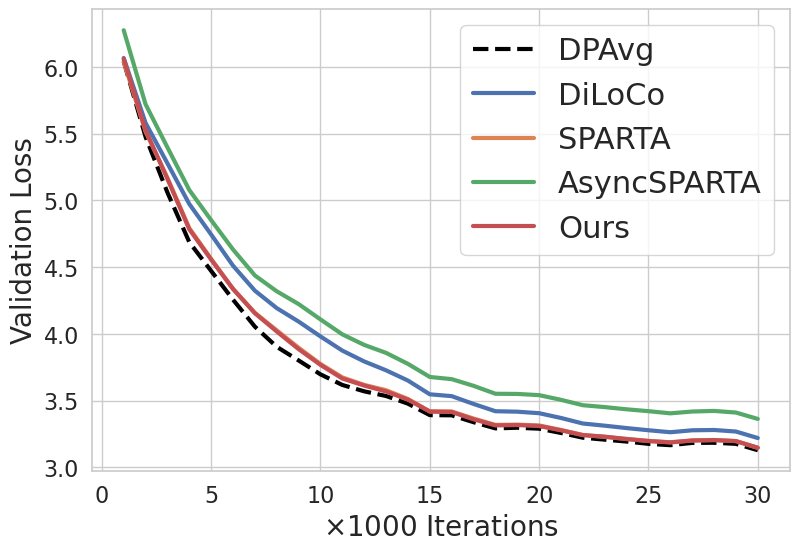}
    \caption{$8\times 2$ mesh}
    \end{subfigure}%
    \begin{subfigure}{0.45\linewidth}
    \includegraphics[width=0.99\linewidth]{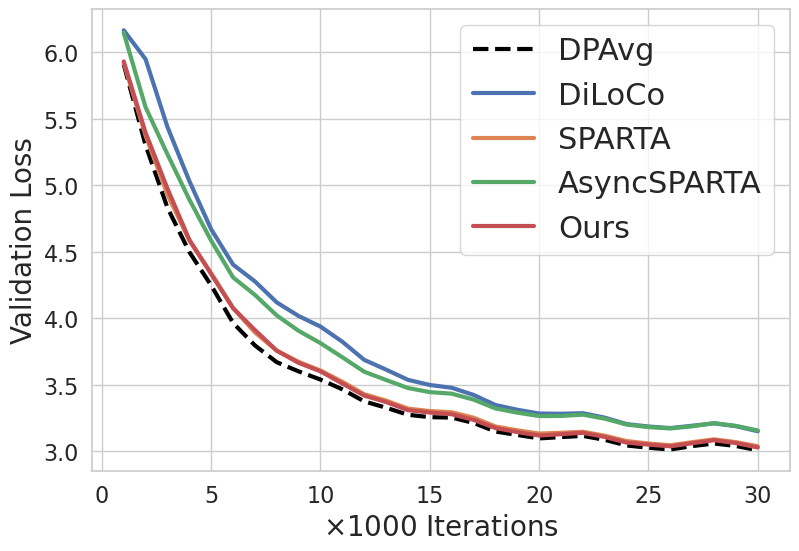}
    \caption{$4\times 2$ mesh, \moe{}}
    \end{subfigure}
    
    \caption{\em Results on \wiki{} with varying mesh configurations and architectures with synchronous \gls{PP}. In all scenarios, our method matches the performance of the fully synchronous method, same as the case with \asyncpp{}.
    }
    \label{fig:main1}
\end{figure}

\begin{figure}[t]
    \centering
    \begin{subfigure}{0.45\linewidth}
    \includegraphics[width=0.99\linewidth]{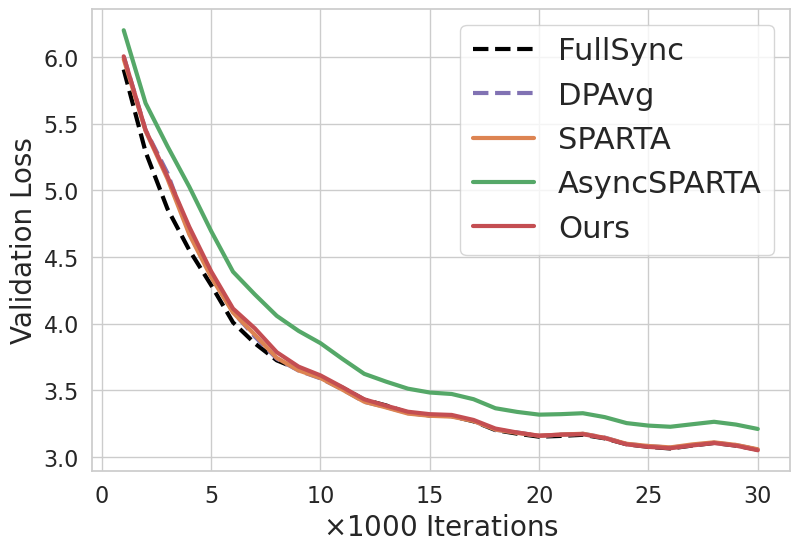}
    \caption{\wiki{}}
    \end{subfigure}%
    \begin{subfigure}{0.45\linewidth}
    \includegraphics[width=0.99\linewidth]{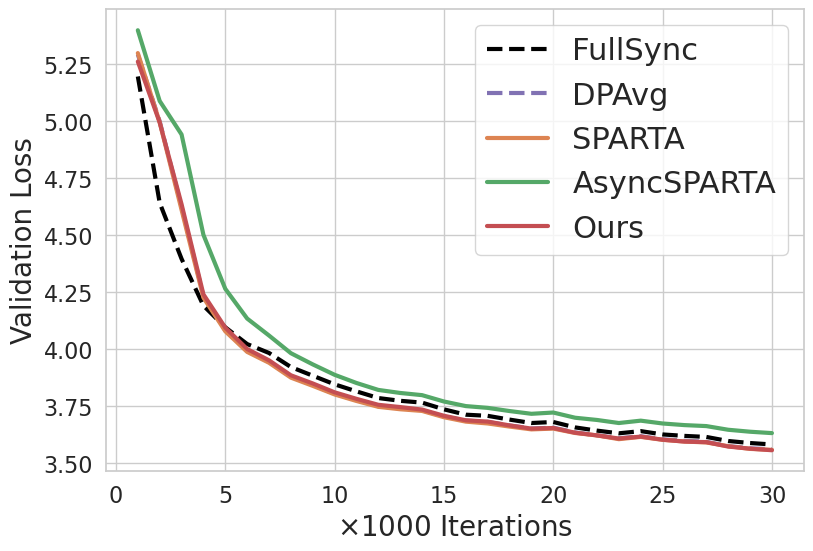}
    \caption{\book{}}
    \end{subfigure}
    \begin{subfigure}{0.45\linewidth}
    \centering
    \includegraphics[width=0.99\linewidth]{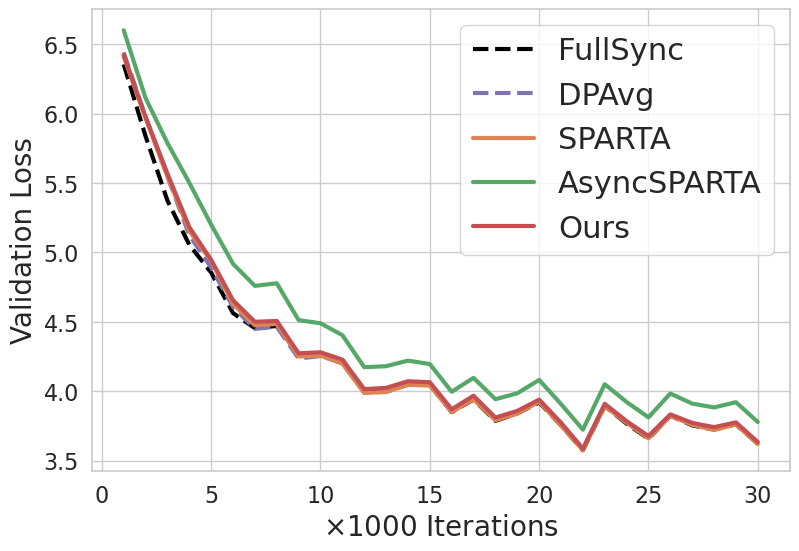}
    \caption{\owt{}}
    \end{subfigure}%
    \begin{subfigure}{0.45\linewidth}
    \centering
    \includegraphics[width=0.99\linewidth]{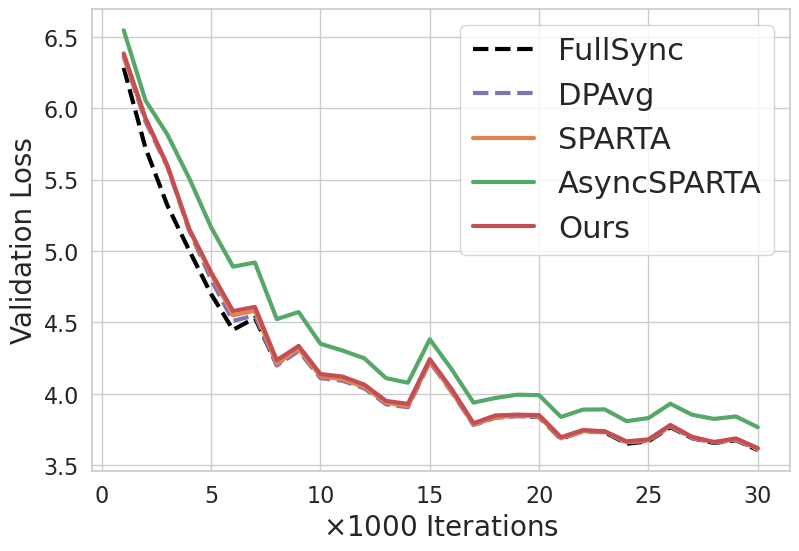}
    \caption{\fw{}}
    \end{subfigure}
    \caption{\em Results on different datasets for $4\times 2$ mesh. Our method performs similarly to FullSync in all scenarios, demonstrating virtually no performance degradation due to staleness or sparse averaging.
    }
    \label{fig:datasets}
\end{figure}

\begin{figure}[t]
    \centering
    \begin{subfigure}{0.45\linewidth}
    \includegraphics[width=0.99\linewidth]{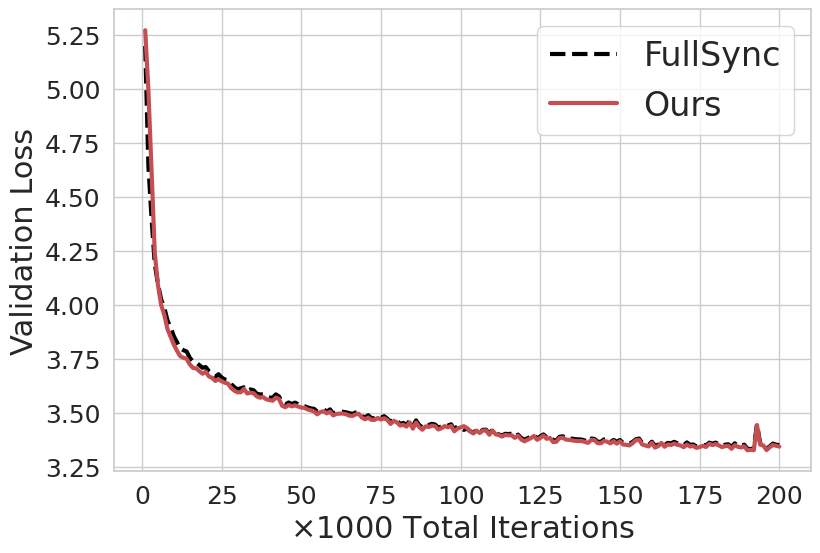}
    \caption{\book{}}
    \end{subfigure}%
    \begin{subfigure}{0.44\linewidth}
    \includegraphics[width=0.99\linewidth]{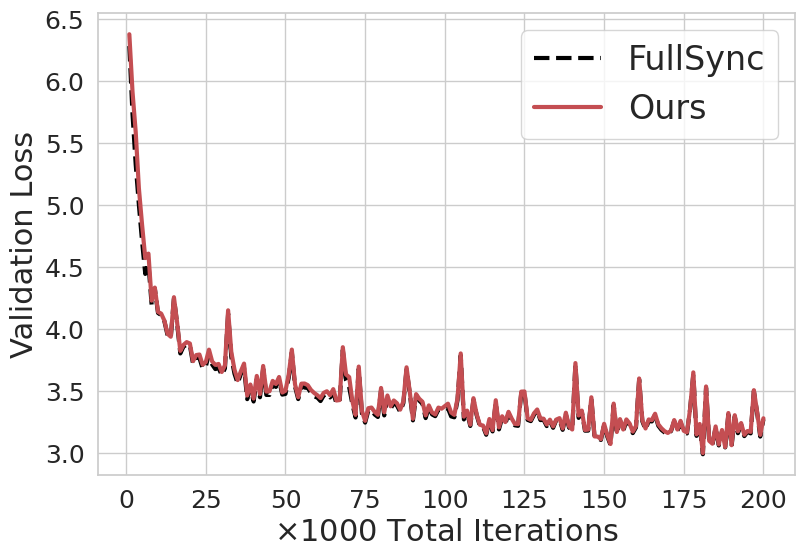}
    \caption{\fw{}}
    \end{subfigure}
    \caption{\em Compute optimal training~\cite{hoffmann2022training} for \book{} and \fw{} for the base model with $4\times 2$ mesh. Our method is nearly identical to FullSync for longer training as well, validating its merits. The final validation perplexities are, for \book{}, {\em FullSync: 28.02}, and {\em Ours: 27.86} and for \fw{}, {\em FullSync: 19.92}, and {\em Ours: 20.10}. The validation curve for \fw{} is noisy for both methods, probably due to the way the validation set is selected from the Huggingface stream.
    }
    \label{fig:co}
\end{figure}

\begin{figure}[t]
    \centering
    \begin{subfigure}{0.45\linewidth}
    \includegraphics[width=0.99\linewidth]{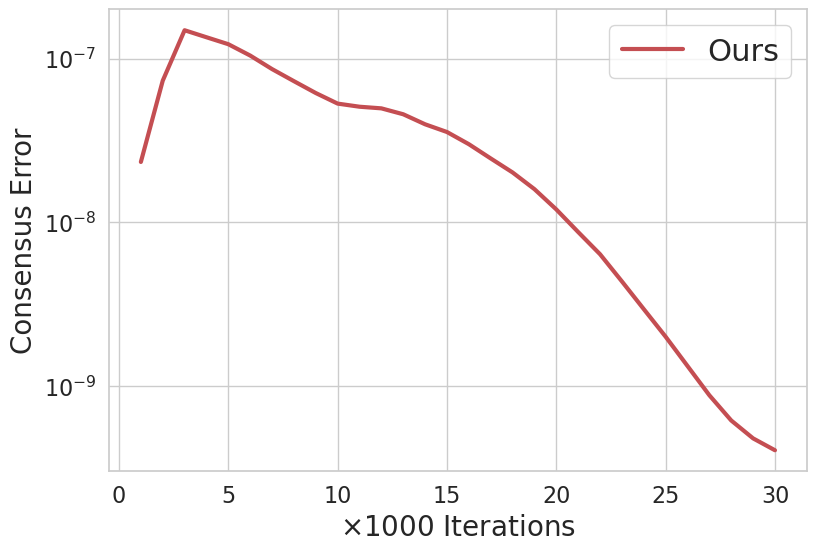}
    \end{subfigure}%
    \begin{subfigure}{0.45\linewidth}
    \includegraphics[width=0.99\linewidth]{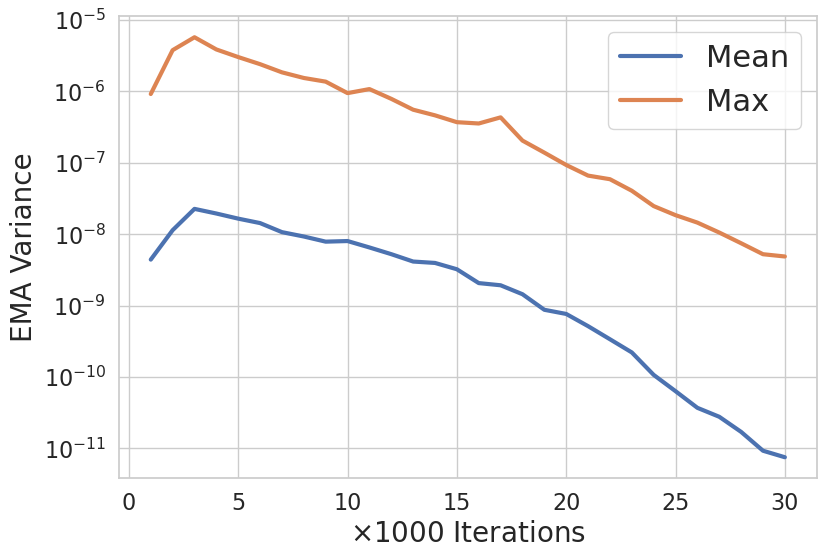}
    \end{subfigure}
    \begin{subfigure}{0.45\linewidth}
    \includegraphics[width=0.99\linewidth]{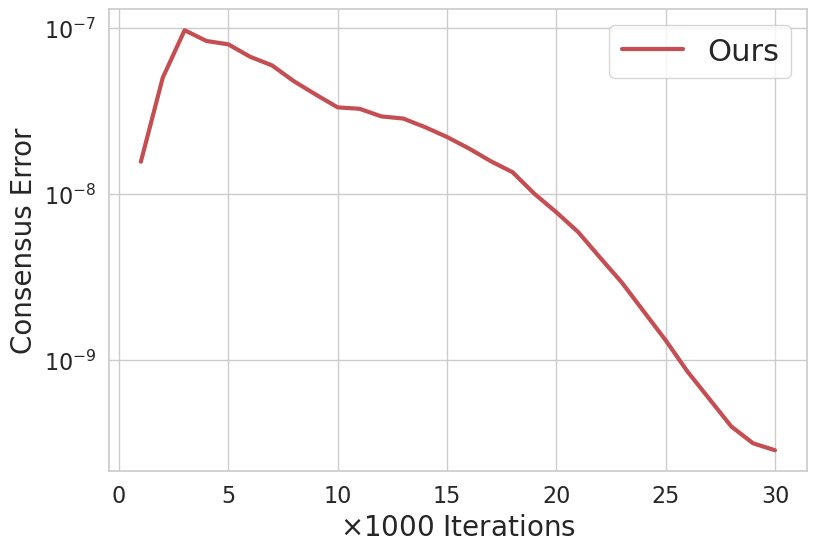}
    \end{subfigure}%
    \begin{subfigure}{0.45\linewidth}
    \includegraphics[width=0.99\linewidth]{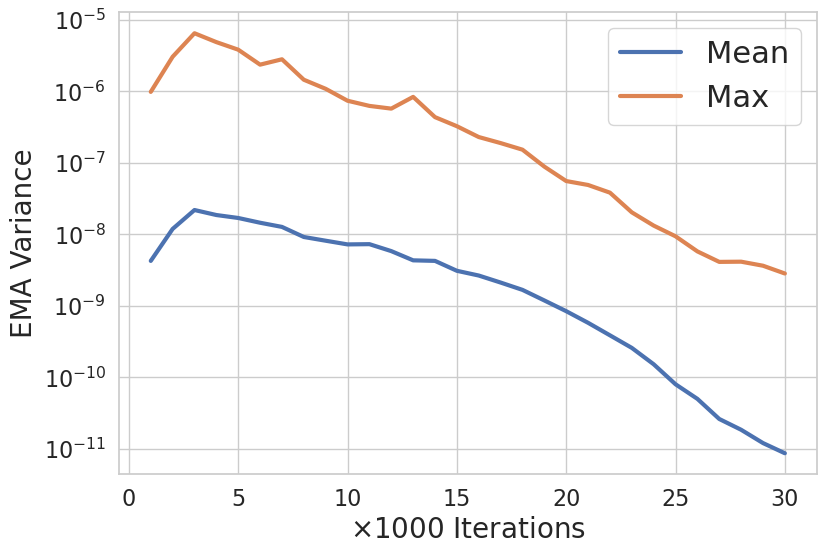}
    \end{subfigure}
    \caption{\em Mean consensus error $\frac{1}{md}\norm{\bfdelta^t}^2$ as in \eqref{eq:notations} and variance between EMA estimates ($\rvd_i^t$) in each replica, for our method on the $2\times 4$ (top) and $4\times 2$ mesh (bottom). For EMA, the mean and max across the model dimension are shown. Results perfectly align with the theory that independent EMA estimates in each replica converge to the expected value (\ie, variance vanishes), and the consensus error for our method diminishes to zero.
    }
    \label{fig:error}
\end{figure}

\begin{figure}[t]
    \centering
    \begin{subfigure}{0.5\linewidth}
    \includegraphics[width=0.99\linewidth]{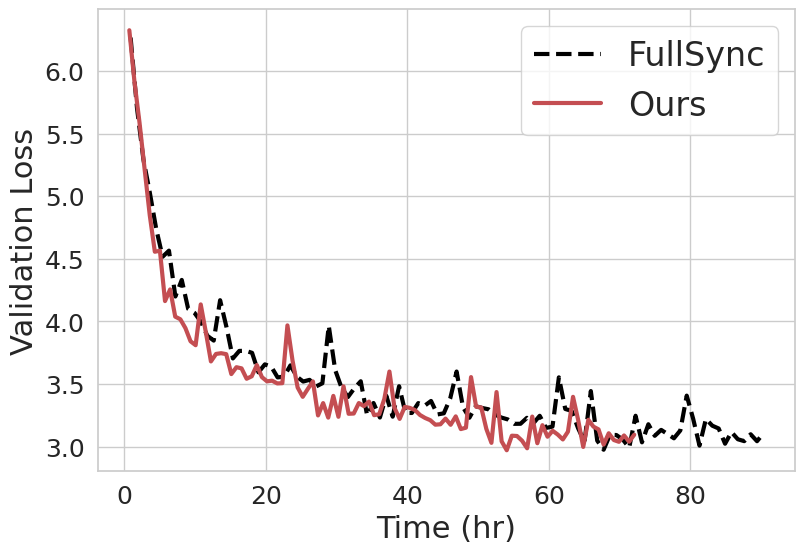}
    \end{subfigure}
    \caption{\em Validation loss vs time for 1B model. Even with suboptimal implementation, fast interconnects, and {\bf not} considering the time gains due to asynchronous updates, FullSync is about 20\% slower than our method.
    }
    \label{fig:1b1}
\end{figure}

\begin{table}   
\centering
\small   
    \begin{tabular}{c|cccc}
        \toprule
        \multirow{2}{*}{PP $\times$ DP Mesh}   & \multirow{2}{*}{AWS Instances} & Communication Time & Forward-Backward time & \multirow{2}{*}{Speed-up} \\
        & &  for tail (ms) &  for tail (ms) &  \\ 
        \midrule
        $4 \times 2$ & $1 \times$ p4d.24 & $515 \pm 35$ & $320 \pm 25$ & $\sim2.6\times$\\
        $4 \times 4$ & $1 \times$ p4d.24 & $920 \pm 150$&$590 \pm 35$ &$\sim2.6\times$\\
        $4 \times 6$ & $2 \times$ p4d.24 & $925 \pm 40$&$580 \pm 15$ &$\sim2.6\times$\\
        $4 \times 8$ & $2 \times$ p4d.24 & $1240 \pm 70$& $600 \pm 30$&$\sim3.1\times$\\
        $4 \times 12$ & $3 \times$ p4d.24 & $1610 \pm 80$& $595 \pm 25$&$\sim3.7\times$\\
        \midrule
        $2 \times 4$ & $1 \times$ p4d.24 & $540 \pm 55$& $470 \pm 35$&$\sim2.2\times$\\
        $4 \times 4$ & $1 \times$ p4d.24 & $920 \pm 150$&$590 \pm 35$&$\sim2.6\times$\\
        $6 \times 4$ & $2 \times$ p4d.24 & $660 \pm 120$&$480 \pm 75$&$\sim2.4\times$\\
        $8 \times 4$ & $2 \times$ p4d.24 & $400 \pm 65$& $815 \pm 35$&$\sim1.5\times$\\
        $12 \times 4$ & $3 \times$ p4d.24 & $735 \pm 65$& $355 \pm 85$&$\sim3.1\times$\\
    
        \bottomrule
    \end{tabular}
    \vspace{1ex}
    \caption{\em We report the \acrshort{SPARTA} communication time (\ie, averaging 5\% of the parameters) for tail stage for the 12-layer base model, that would be masked by our asynchronous sparse averaging for various configurations above and compare it with the forward-backward times. This shows the empirical speed-up that could be obtained by our method, which ranges from $1.5\times$ -- $3.7\times$ and improves with larger mesh. This speed-up only considers asynchronous \gls{DP} and any benefits due to \asyncpp{} is additional to this.
    }
    \label{tab:time}
\end{table}


\begin{figure}[t]
    \centering
    \begin{subfigure}{0.32\linewidth}
    \includegraphics[width=0.99\linewidth]{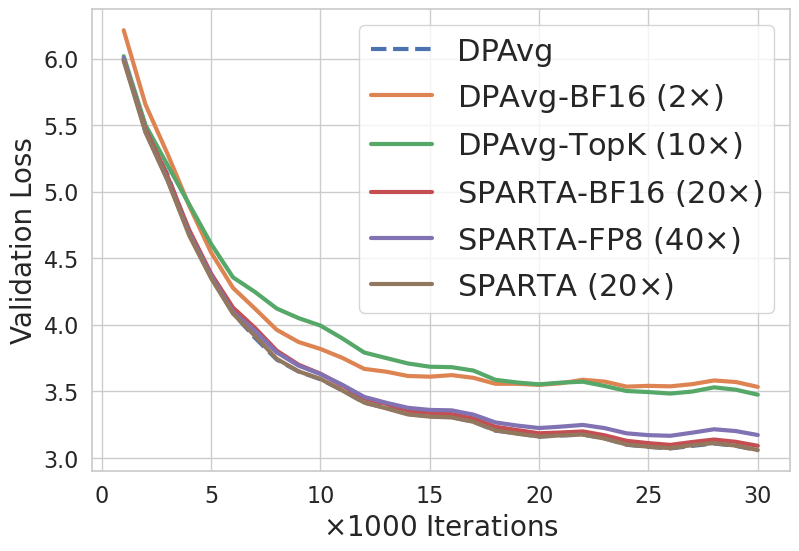}
    \end{subfigure}%
    \begin{subfigure}{0.32\linewidth}
    \includegraphics[width=0.99\linewidth]{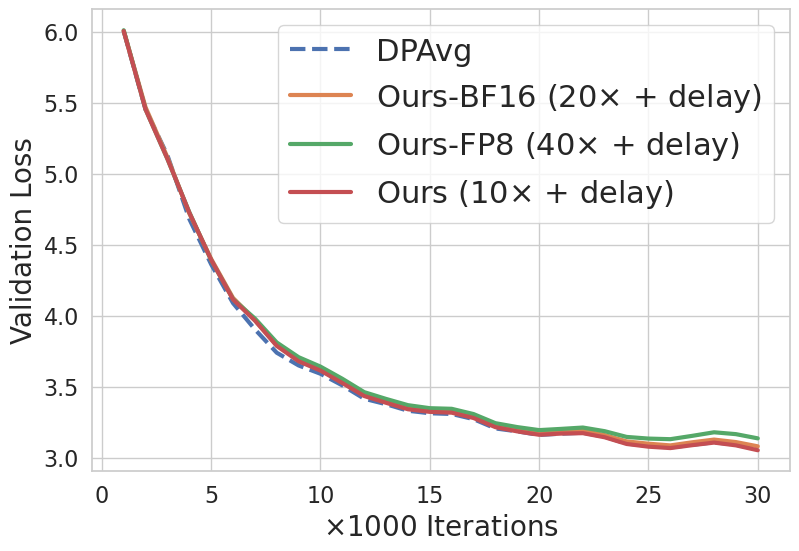}
    \end{subfigure}%
    \begin{subfigure}{0.33\linewidth}
    \includegraphics[width=0.99\linewidth]{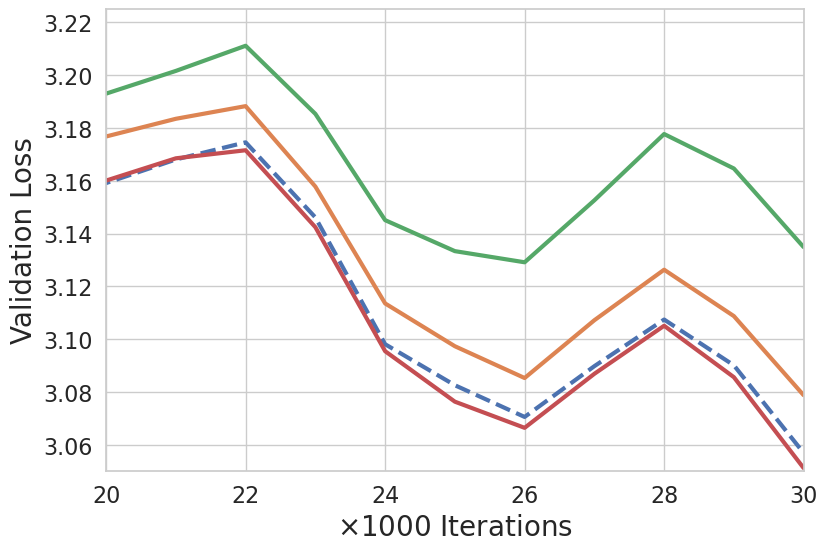}
    \end{subfigure}
    \caption{\em Effect of compression for different methods, for the base model with $4\times 2$ mesh on \wiki{}. {\bf Left:} Both quantisation and TopK sampling based on weight magnitude (instead of random) degrade the performance for \DP{} (\DP{}-{\em FP8} did not converge). However, SPARTA is robust to quantization. {\bf Right: } Similar to SPARTA, our method is robust to quantization even with a 10-step delay, and the degradation is minimal. TopK sampling did not converge for our method, aligning with our insight that the sampling needs to be unbiased. \DP{} is even more sensitive to quantization with delay, and \DP{}-{\em BF16} with delay did not converge.    
    The robustness of sparse averaging (even with delay) to quantization is intriguing, and it may be explained by the fact that since the quantization error is introduced only for a small subset (5\% in this case) at each iteration, the effect of quantization on training is negligible. However, this warrants further study, which is beyond the scope of this work.
    }
    \label{fig:quant1}
\end{figure}

\begin{figure}[t]
    \centering
    \begin{subfigure}{0.32\linewidth}
    \includegraphics[width=0.99\linewidth]{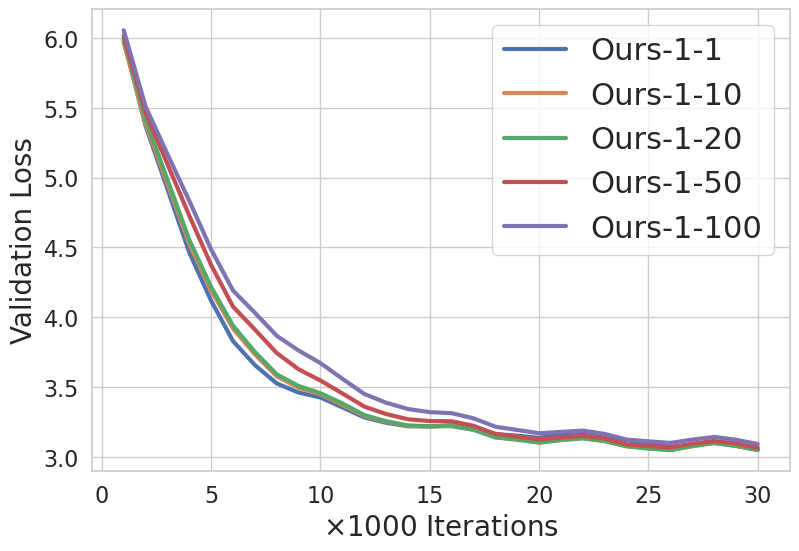}
    \end{subfigure}%
    \begin{subfigure}{0.33\linewidth}
    \includegraphics[width=0.99\linewidth]{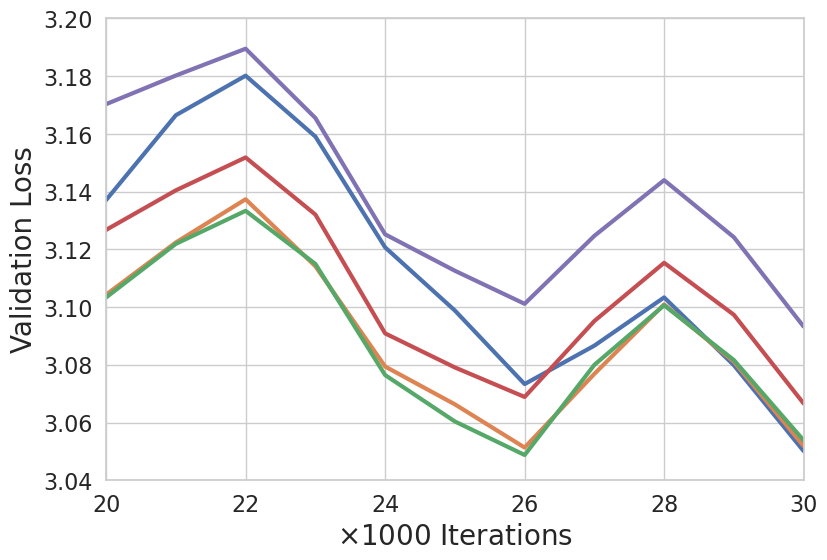}
    \end{subfigure}%
    \begin{subfigure}{0.32\linewidth}
    \includegraphics[width=0.99\linewidth]{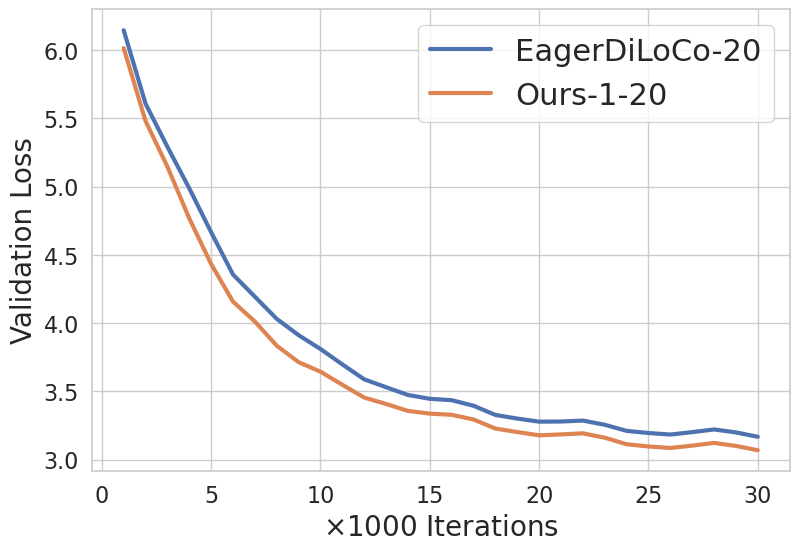}
    \end{subfigure}
    \caption{\em Results in an equivalent setup to Eager\diloco{}~\cite{kale2025eager}. Ours-1-X denotes our method with subset size 1 (full averaging) and varying averaging interval (\ie, delay for asynchronous \gls{DP}). On the right, we compare against Eager\diloco{} for 20 inner-steps (\ie, delay). Our EMA approach outperforms eager updates, confirming the strict generality, and the performance varies only slightly with different numbers of inner steps.
    }
    \label{fig:eager}
\end{figure}